\documentclass[10pt,oneside,letterpaper]{article}
 \pdfoutput=1
\usepackage{epsfig}
\usepackage[usenames]{color}
\usepackage{amsmath}
\usepackage{amssymb}
\usepackage{amsthm}
\usepackage{mathrsfs}
\usepackage{amsfonts}
\usepackage{lscape}
\usepackage{graphicx}
\usepackage{enumerate}
\usepackage{natbib}
\usepackage{appendix}
\usepackage{hyperref}
\theoremstyle{plain}
\newtheorem{thm}{Theorem}[section]
\newtheorem{cor}[thm]{Corollary}
\newtheorem{lem}[thm]{Lemma}

\newcommand{\eqq}[1]{\begin{equation} #1 \end{equation}}
\newcommand{\alge}[1]{\begin{align*} #1 \end{align*}}
\newcommand{\algg}[1]{\begin{align} #1 \end{align}}

\newcommand{\C}{\mathcal{C}}
\newcommand{\Chi}{\mathcal{X}}

\newcommand{\trn}{\scriptscriptstyle\mathrm{T}}

\newcommand{\Zcal}{\mathcal{Z}}

\newcommand{\Pcal}{\mathcal{P}}
\newcommand{\Ecal}{\mathcal{E}}
\newcommand{\PP}{\mathbb{P}}

\newcommand{\E}{\mbox{E}}
\newcommand{\PS}{\mbox{P}}
\newcommand{\me}{\mathcal{E}}
\newcommand{\ppp}{|\!|\!|}
\newcommand{\ddd}{d}
\newcommand{\sprs}{\eta}
\newcommand{\maxcon}{\nu_1}
\newcommand{\maxinf}{\nu}
\newcommand{\maxcorrlone}{\omega}
\newcommand{\maxcorr}{\omega_0}
\newcommand{\inftau}{\rho}

\newcommand{\Rbb}{\mathbb{R}}

\newcommand{\perr}[1]{p_{err,\,#1}}
\newcommand{\perrp}[1]{\tilde p_{err,\,#1}}
\newcommand{\trt}{\tilde r_2}
\newcommand{\fit}{\hat U}

\begin{document}

\title{Variable Selection in High Dimensions with Random Designs and Orthogonal Matching Pursuit}
\author{Antony Joseph\thanks{Department of Statistics, Yale University , New Haven, CT 06520 USA, e-mail : antony.joseph@yale.edu}}
\date{August 29, 2011}
\maketitle


\begin{abstract}
The performance of Orthogonal Matching Pursuit (OMP) for variable selection is analyzed for random designs. When contrasted with the deterministic case, since the performance is here measured after averaging over the distribution of the design matrix, one can have far less stringent sparsity constraints on the coefficient vector.  We demonstrate that for exact sparse vectors, the performance of the OMP is similar to known results on the Lasso algorithm [\textit{IEEE Trans. Inform. Theory} \textbf{55} (2009) 2183-–2202]. Moreover, variable selection under a more relaxed sparsity assumption on the coefficient vector, whereby one has only control on the $\ell_1$ norm of the smaller coefficients, is also analyzed. As a consequence of these results, we also show that the coefficient estimate satisfies strong oracle type inequalities.

\end{abstract}

\section{Introduction}\label{sec:intro}
Consider linear regression model,
\eqq{Y = X\beta + \epsilon \label{eqone}}
where $X \in \Rbb^{n\times p}$, the coefficient vector $\beta \in \Rbb^p$ and noise $\epsilon \in \Rbb^n$.
The high dimensional case, where $p$ is of the same order, or possibly much larger than $n$, has been of immense interest nowadays. In many applications, interest is not primarily on prediction of the response $Y$, but on the accuracy of estimation of the coefficient $\beta$. Examples of such applications include, micro-array data analysis, graphical model selection \cite{meinshausen2006high}, compressed sensing \cite{donoho2006compressed}, \cite{candes2006near}, and in communications \cite{barron2010joseph},\cite{barron2010ajoseph},\cite{tropp2006just}. As is well known, in the high dimensional setting,  $\beta$  is unidentifiable unless the design matrix $X$ is well-structured and there is some sparsity constraint on the coefficient vector $\beta$. This sparsity assumption corresponds to restricting $\beta$ to few non-zero entries ($\ell_0$-sparsity), or more generally, assuming that  $\beta$ has only few terms that are large in magnitude.

The Orthogonal Matching Pursuit \cite{pati1993orthogonal} is a variant of the Matching Pursuit algorithm \cite{zhang1993matching}, where, successive fits are computed through the least squares projection of $Y$ on the current set of selected terms.
For deterministic $X$ matrices, variable selection properties of this algorithm, for $\ell_0$-sparse vectors, have been analyzed for the noisy case in \citet{zhang2009consistency} and \citet{cai2010orthogonal}. However, as we shall review Subsection \ref{subsec:relatedwork}, although they give strong performance guarantees under certain conditions on the X matrix, they impose severe constraints on the sparsity of $\beta$. 
Similar results have been shown for the Lasso, for example in \citet{zhao2006model}.
With random designs one can have reliable detection of the support with far less stringent sparsity constraints; the performance is here measured after averaging over the distribution of $X$. 
For example, \citet{wainwright2009sharp} proved such results for the Lasso algorithm. The main results of this paper, apart from showing that similar properties hold for the OMP, demonstrate two important additional properties. Firstly, we give results on partial support recovery, which is important since exact recovery of support places strong requirements on $n$ if some of the non-zero elements are small in magnitude. Secondly, and more importantly, we relax the assumption that $\beta$ is $\ell_0$-sparse and address variable selection under a more general notion of sparsity, whereby one has only control on the $\ell_1$ norm of the smaller elements of $\beta$.  We demonstrate that even under this more relaxed assumption, one can reliably estimate the position of the larger entries using the OMP. This has certain parallels with recent work on the Lasso by \citet{zhang2008sparsity}. As a consequence of these results, we show that our coefficient estimate, after running the algorithm, satisfies strong oracle inequalities, similar to that demonstrated for the Lasso \cite{zhang2009some} and Dantzig selector \cite{candes2007dantzig}.

The paper is organized as follows. Below, we describe the OMP algorithm. The stopping criterion we use is slightly different from what is traditionally used in literature.  Subsection \ref{subsec:relatedwork} motivates in greater detail our interest in random designs. In Subsection \ref{subsec:subgaus} we give results for design matrices that have i.i.d sub-Gaussian entries and $\ell_0$-sparse vectors. This extends the results in \citet{tropp2007signal} for the noisy case. In Subsection \ref{subsec:gauss} we describe more general results with correlated Gaussian designs, where we only have control over the $\ell_1$ norm of the smaller coefficients. Sections \ref{sec:proofsubgauss}, \ref{sec:proofguass} and \ref{pf:sec:proofgauss} gives proofs of our main results. The appendices contains auxiliary results.

\subsection{The Orthogonal Matching Pursuit algorithm} \label{subsec:oga}

Denote as $J = J_1 = \{1,\, 2,\, \ldots,\, p\}$ to be the set of indices corresponding to columns in the $X$ matrix.
For each step $i$, with $i \geq 1$, a single index $a(i)$ is detected to be non-zero in that step. Accordingly, denoting $d(i) = a(1) \cup a(2) \ldots \cup a(i)$ as the set of detected columns after $i$ steps, step $i + 1$ of the algorithm only operates on the columns in $J_{i+1} = J - d(i)$, that is, the columns not detected in the previous steps. In other words, indices detected in previous steps remain detected.

The decision on whether a particular index $j$ is detected during a particular step $i$ is based on the absolute value of a statistic $\Zcal_{ij}$. Here, $\Zcal_{ij}$ is simply the inner product between $X_j$ and the normalized residual $ R_{i-1}$ computed for the previous step.

Apart from the response vector $Y$ and design matrix $X$, the other input to the algorithm is a positive threshold value $\tau$. Denote $\|.\|$ as the euclidean norm. We now describe the OMP algorithm.

\begin{itemize}
\item Initialize $ R_0 = Y,\, d(0) = \emptyset$. Start with step $i = 1$.
\item Update
$$\Zcal_{ij} = X_j^{\trn} \frac{ R_{i-1}}{\| R_{i-1}\|}, \quad \mbox{for} \quad j \in J_i.$$

\item If $\max_{j \in J_i}|\Zcal_{ij}| > \tau$, do the following:
\begin{itemize}
\item Assign
$a(i) = \arg\max \{|\Zcal_{i,j}| : j \in J_i\}.$
\item Set $d(i) = d(i-1) \cup a(i)$. Update $ R_i = (I - \Pcal_{i})Y$, where $\Pcal_{i}$ is the projection matrix for the column space of $X_{d(i)}$, and set $J_{i+1} = J_i - a(i)$.
\item Increase $i$ by one and go to step 2.
\end{itemize}
\item Stop if $\max_{j \in J_i} |\Zcal_{ij}| \leq \tau$.

\end{itemize}

We remark that for any step $i$, the inner product $X_j^{\trn}R_{i-1}$, for $j \in d(i-1)$, is 0. Correspondingly, since $\Zcal_{ij} = 0$, for $j \in d(i-1)$, the maximum of $\Zcal_{ij}$ over $j \in J_i$, is the same as the maximum over all $j \in J$. Also, the newly selected term $a(i)$ may be equivalently expressed as,
$$a(i) = \arg\min_{j \in J} \inf_{w \in \Rbb} \|Y - Fit_{i-1} - w X_j\|^2,$$
where $Fit_{i-1}$ is the least squares fit of $Y$ on the columns in $d(i-1)$. In this respect, the OMP is similar to other greedy algorithms such as relaxed greedy and forward-stepwise algorithms (\cite{barron2008approximation}, \cite{huang2008risk}, \cite{jones1992simple}, \cite{lee1996efficient}), that operate through successive reduction in the approximation error.

 As mentioned earlier, the stopping criterion considered here is slightly different from that considered in literature.
Traditionally, for the no noise setting, the algorithm is run until there is a perfect fit between $Y$ and the selected terms, that is $R_i =0$ (see for example \cite{tropp2004greed}, \cite{tropp2007signal}). In the noisy case, as analyzed over here, there are two standard approaches. The first, as done in \cite{cai2010orthogonal}, \cite{zhang2009consistency}, is to stop when $\max_{j \in J}|X_j^{\trn}R_{i-1}|$ is less than some fixed threshold. The second approach, as analyzed in \cite{donoho2006stable}, \cite{cai2010orthogonal}, is to stop when $\|R_{i}\|$ is less than some pre-specified value.

Our stopping criterion, which is more similar to the first approach, is equivalent to continuing the algorithm until $\max_{j \in J}|X_j^{\trn}R_{i-1}| \leq \tau \|R_{i-1}\|$. The motivation for the use of such a statistic comes from the analysis of a similar iterative algorithm in \citet{barron2010ajoseph} for a communications setting. However, there the values of the non-zero $\beta_j$'s were known in advance; this added information played an important role in the analysis of the algorithm. A similar statistic was used by \citet{fletcherorthogonal} for an asymptotic analysis of the OMP for exact support recovery using i.i.d designs.

 \textit{Notation}: Let $a = a(n,\,p,\,  k),\, b = b(n,\, p,\,  k)$ be two positive functions of $n,\, p$ and $k$. We denote as $a = O(b)$, if $a \leq c_1 b$ for some constant positive constant $c_1$ that is independent of $n,\, p$ or $k$. Similarly, $a = \Omega(b)$ means $a \geq c_2 b$ for positive $c_2$ independent of $n,\, p$ or $k$.

\subsection{Related work}\label{subsec:relatedwork}

As mentioned earlier, we are interested in  variable selection in the high dimensional setting. Apart from iterative schemes, another popular approach is the convex relaxation scheme  Lasso \cite{tibshirani1996regression}.
In order to motivate our interest in random design matrices, we describe existing results on variable selection, using both methods, with deterministic as well as random design matrices. For convenience, we concentrate on implications of these results assuming the simplest sparsity constraint on $\beta$, namely that $\beta$ has only a few non-zero entries.

In particular, we assume that,
\eqq{|S_0(\beta)| = k, \quad \mbox{where\, $S_0(\beta) = \{j : \beta_j \neq 0\}$.}\label{snotbeta}}
In other words, attention is restricted to all $k$-sparse vectors, that is, those that have exactly $k$ non-zero entries. 
For convenience, we drop the dependence on $\beta$ and denote $S_0(\beta)$ as  $S_0$ whenever there is no ambiguity.
The simplest goal then is to recover $S_0$ exactly, under the additional assumption that all $\beta_j$, for $j \in S_0$, have magnitude at least $\beta_{min}$, where $\beta_{min} >0$. Denote as $\C \equiv \C(\beta_{min},\, k)$, as the set of coefficient vectors satisfying this assumption.

Further, denote $\hat S$  as the estimate of $S_0$ obtained  using either method, and  $\Ecal =\{\hat S \neq S_0\}$  the error event that one is not able to recover the support exactly. For deterministic $X$, interest is mainly on conditions on $X$  so that
\eqq{P_{err,\, X} = \sup_{\beta \in \C }\PP_{\beta}\left(\Ecal |X\right)\label{uniformrecov}}
can be made arbitrarily small when $n,\, p,$ or $k$ become large. Here $\PP_{\beta}(.|X)$ denotes the distribution of $Y$ for the given $X$ and $\beta$.


A common sufficient condition on $X$ for this type of recovery is the \textit{mutual incoherence condition}, which requires that the the inner product between distinct columns be small.   In particular, letting $\|X_j\|^2/n = 1$, for all $j \in J$, it is assumed that \eqq{\gamma(X) = \frac{1}{n}\max_{j \neq j'}\left|X_j^{\trn}X_{j'}\right| \label{incoherence}}  is $O(1/k)$. Another related criterion is the \textit{irrepresentable criterion} \cite{tropp2004greed}, \cite{zhao2006model}, which assumes, for all subset $T$ of size $k$, that
 \eqq{\|(X_{T}^{\trn}X_{T})^{-1}X_{T}^{\trn}X_j\|_1 < 1, \quad\mbox{for all} \quad j \in J - T \label{irrepresentablecondpop}.}
Here $\|.\|_1$ denotes the $\ell_1$ norm.

Observe that if $P_{err, X}$ (\ref{uniformrecov}) is small, it
 gives  strong guarantees on support recovery, since it ensures that any $\beta$, with $|S_0(\beta)| = k$, can be recovered with high probability. 
 However, it imposes severe constraints on the $X$ matrix. As as example, when the entries of $X$ are i.i.d Gaussian, the coherence $\gamma(X)$ is around $\sqrt{2\log p/n}$. Correspondingly, for (\ref{incoherence}) to hold, $n$ needs to be $\Omega(k^2\log p)$. In other words, the sparsity $k$ should be $O(\sqrt{n/\log p})$, which is rather strong since ideally one would like $k$ to be of the same order as $n$. Similar requirements are needed for the irrepresentable condition to hold. 
Recovery using the irrepresentable condition has been shown for Lasso in \cite{zhao2006model}, \cite{wainwright2009sharp}, and for the OMP in \cite{zhang2009consistency}, \cite{cai2010orthogonal}.
Indeed, it has been observed, in \cite{zhao2006model} for the Lasso, and in \cite{zhang2009consistency}, for the OMP, that a similar such condition is also necessary if one wanted exact recovery of the support, while keeping $P_{err,\, X}$ small.


A natural question is to ask  about requirements on $X$ to ensure recovery in an average sense, as opposed to the strong sense described above. One way to proceed, as done over here, is to consider random $X$ matrices and ask about the requirements on $n,\, p,\, k$, as well as $\beta_{min}$, so that
\eqq{P_{err} = \sup_{\beta \in \C}\PP_{\beta}\left(\Ecal \right)\label{avgXsense}}
is small. Here $\PP_{\beta}\left(\Ecal \right) = E_X \PP_{\beta}\left(\Ecal |X\right)$, where the expectation on the right is over the distribution of $X$. For the Lasso, \citet{wainwright2009sharp} considers random $X$ matrices, with rows drawn i.i.d $N_p(0, \Sigma)$. 
 It is shown that under certain conditions on $\Sigma$, which can be described as population counterparts of the conditions for deterministic $X$'s, one can recover $S_0$ with high probability with $n = \Omega(k\log p)$ observations, with the constant depending inversely on $\beta_{min}^2$. The form of $n$ is in a sense ideal since now $k = O(n/\log p)$ is nearly the same $n$, if we ignore the $\log p$ factor. 
 As mentioned earlier, apart from establishing similar properties to hold for the OMP with $k$-sparse vectors, we also demonstrate strong support recovery results under a more general notion of sparsity. These results are described in the next section.

We also note that instead of averaging over $X$, one could assume a distribution on $\beta$ and analyze the average probability of $\Ecal$ over this distribution. This is done in \citet{candès2009near} for the Lasso. Here, for fixed magnitudes of the $k$ non-zero $\beta$, the support of $\beta$ is uniformly assigned over all possible subsets of size $k$. Once the support is chosen, the signs for the non-zero $\beta_j$'s  are assigned $\pm 1$ with equal probability. 
If $\mbox{Avg}[.]$ denotes the expectation with this distribution of $\beta$, it is shown that
one could keep $\mbox{Avg} \left[\PP_{\beta}\left(\Ecal |X\right)\right]$ low  for $\gamma(X)$ as high as $O(1/\log p)$. This condition on $\gamma(X)$ is less stringent than before and leads to a demonstration that $n =\Omega(k\log p)$ is sufficient for support recovery, provided $\ppp X\ppp \approx \sqrt{p/n}$, where $\ppp .\ppp$ denotes the spectral norm. 
We provide comparisons with this work in Section \ref{sec:conclusion}.


\textit{Notation:} For a set $\mathcal{A} \subseteq J$, we denote as $X_{\mathcal{A}}$  the sub-matrix of $X$ comprising of columns with indices in $\mathcal{A}$. Similarly, for any $p \times 1$ vector $ \beta$, we denote as $\beta_{\mathcal{A}}$ the $|\mathcal{A}|\times 1$ sub-vector with indices in $\mathcal{A}$. Also let $\mathcal{A}^c = J - \mathcal{A}$.

\section{Results}\label{sec:results}

Before discussing our main results with Gaussian matrices, in Subsection \ref{subsec:subgaus} we state results when the entries of $X$ are i.i.d sub-Gaussian and when the vector $\beta$ has $k$ non-zero entries. 
The noise vector is also assumed to come from a sub-Gaussian distribution with scale $\sigma$. 
This generalizes the results of  \citet{tropp2007signal} for the noisy case. While preparing this manuscript we discovered that \citet{fletcherorthogonal} have analyzed the OMP for i.i.d designs and for $k$-sparse vectors, similar to that in Subsection \ref{subsec:subgaus}. However, there the analysis was for exact support recovery and was asymptotic in nature. Further, they focused on a specific regime, where $k\beta_{min}^2/\sigma^2$ tends to infinity. We provide more comparisons with this work later on in the paper.

We show that $n =\Omega(k\log p)$ samples are sufficient for the recovery of any coefficient vector with $\beta_{min}$ that is at least the same order as the \textit{noise level}. More specifically, 
define
\eqq{\mu_n = \sqrt{(2\log p)/n}\label{mundef}.}
The quantity $\sigma\mu_n$ can thought of as the noise level. To see why this is so, consider the orthogonal design where $X^{\trn}X/n = I$ and noise $\epsilon \sim N(0, \sigma^2 I)$. Assume that, as usual, we are interested in recovering any $\beta$ with $|S_0(\beta)| = k$. A natural estimate of the support would be,
\eqq{\hat S = \{j : |z_j| > t\} \quad\mbox{with}\quad z_j = X_j^{\trn}Y/n,\label{zjhats}}
where $t$ is positive. Notice that $z_j \sim N(\beta_j, \sigma^2/n)$ for each $j \in J$. Correspondingly, since $z_j \sim N(0, \sigma^2/n)$, for $j \in  J - S_0$, one sees that $t$ has to be of the form $\sigma\mu_n$ in order to prevent false discoveries with high probability. Similarly $\beta_j$, for all $j \in S_0$, has to have magnitude at least $\sigma\mu_n$ if one wanted to avoid false negatives.

The analysis of iid designs, as done in Subsection \ref{subsec:subgaus}, forms an important ingredient to compressed sensing \cite{candes2006near}, \cite{donoho2006compressed}. However, it may not be useful for statistical applications, where typically the choice of the $X$ matrix is not under ones control. Accordingly, in Subsection \ref{subsec:gauss}, we assume that the rows of $X$ are drawn i.i.d from $N_p(0, \Sigma)$, with certain assumptions on $\Sigma$. This model was also employed to detect the neighborhood of a node in high dimensional graphs by \citet{meinshausen2006high}.
Moreover, we relax the assumption that $\beta$ is $k$-sparse and only assume that there is a set $S = S(\beta)$, of size $k$, such that $\beta_{S^c} $ is sparse in a more general sense. Here $\beta_{S^c}$ denotes the vector of coefficients outside of $S$. 
More specifically, for a constant $\maxinf \geq 0$, if
\eqq{ S = \left\{j : |\beta_j| > \sigma\maxinf\mu_n\right\},\quad\mbox{with}\quad |S| = k\label{sassump},}
we assume
\eqq{\|\beta_{S^c}\|_1 \leq \sigma\sprs\mu_n \label{betascl1},}
for an appropriately chosen $\sprs$. A natural choice would be to take $\maxinf = 1$. Then, $S$ would correspond to the indices above the noise level. We show that for $\eta$ not too large, the OMP can detect the large indices in $S$ with high probability, provided $\Sigma$ satisfies certain conditions.
 As a consequence of these results, we show that the coefficient estimate satisfies strong oracle inequalities.


\subsection{Recovery with  sub-Gaussian designs} \label{subsec:subgaus}

In  this section we address the requirements on $n,\, p,\, k$ as well as $\beta_{min}$, to recover the support of $\beta$, either exactly or nearly so, where we assume that $|S_0(\beta)| = k$. Here $S_0(\beta)$ is as in (\ref{snotbeta}). We allow the case that $k$ may be zero. Further, since it may not be a realistic assumption that $k$ is known, we assume that we only know an upper bound $\bar k$ on $k$, with $\bar k \geq \max\{k,\,1\}$.

Let $X_{\ell j}$, for $\ell = 1,\, \ldots,\, n$ and $j =1,\, \ldots,\, p$, denote the entries of the $X$ matrix. Throughout this section we assume that the $X_{\ell j}$'s are independent sub-Gaussian with mean 0 and scale $1$, that is $\E e^{tX_{\ell j}} \leq e^{t^2/2}$, for  $t \in \Rbb$. Further, we assume that the noise vector $\epsilon$ is independent of $X$ and has independent sub-gaussian entries with mean 0 and scale $\sigma$, that is
$\E e^{t\epsilon_{\ell}} \leq e^{\sigma^2 t^2/2}$, for  $t \in \Rbb$, $\ell = 1,\, \ldots,\, n$. Additionally, if $k \geq 1$, we assume that the following two conditions are satisfied with high probability.
\begin{description}
\item[Condition 1.]  There exists $\lambda_{max} \geq \lambda_{min} > 0$, so that the eigenvalues of $X_{S_0}^{\trn}X_{S_0}/n$
are between $\lambda_{min}$ and $\lambda_{max}$, that is $$\lambda_{max}\|v\|^2 \geq \|X_{S_0} v\|^2/n \geq \lambda_{min}\|v\|^2 \quad \mbox{for all} \quad v\in \Rbb^{k}.$$
\item [Condition 2.] The $\ell_2$ norm of the noise vector is bounded, that is $\|\epsilon\|^2/n \leq  \sigma^2\lambda$, for some $\lambda >0$.
\end{description}
 Let $\mathcal{E}_{cond}$ be the event that Conditions 1 or 2 fail. 
The first assumption is related to the restricted isometry property (\citet{candes2005decoding}) and the sparse eigenvalues conditions  (\citet{zhang2008sparsity}).
Condition 1 is satisfied for a wide variety of random ensembles. For example, it is satisfied with high probability for the Gaussian ensemble, where the $X_{\ell j}$ are i.i.d $N(0,1)$ and the binary ensemble, where the $X_{\ell j}$ are i.i.d uniform on $\{-1,\, +1\}$ (see for example, \citet{baraniuk2007johnson}). 
Notice that since we are interested in controlling the probability $P_{err}$ in (\ref{avgXsense}), because of the averaging over $X$, we do require that the Condition 1 hold uniformly over all $S_0$, with $|S_0| = k$. Condition 2, which bounds the $\ell_2$ norm of the noise vector, is required for controlling the norm of the residuals $R_i$. It is satisfied with high probability, for example, when the noise $\epsilon \sim N(0, \sigma^2)$.

Below, we state the theorem giving sufficient conditions on $n$ for reliable recovery of the support of $\beta$. The threshold $\tau$ is taken to be
\eqq{\tau = \sqrt{2(1 + a)\log p}\label{tauformintro},}
for some $a >0$. 
Here $n$ will be a function $\bar k$ and $p$, as well as the various quantities defined above. The results of course hold with $\bar k$ replaced by $k$, provided $k$ is non-zero. 
In particular, for $\alpha,\,\delta > 0$, define \eqq{\xi \equiv \xi(\alpha,\, \delta) = \max\left\{(1 + \delta)r_1,\, \sigma^2 r_2^2 f(\delta)/(\bar k\alpha)\right\}\label{xidef}.}
where,
\eqq{r_1 = \frac{\max\{\lambda_{max},\, \lambda\}}{\lambda_{min}^3} \quad\quad r_2 = \left[\frac{1}{\sqrt{\lambda_{min}}} + \sqrt{r_1}\right]\label{equationr1r2}
}
and
\eqq{f(\delta) = \frac{1}{\left(1 - 1/\sqrt{1 + \delta}\right)^2} \label{equationfdel}}

Denote as $\hat S = \hat S(Y, X,\, \tau)$, the estimate of the support obtained after running the algorithm with the given $Y,\, X$ and threshold $\tau$. Further, denote the undetected elements of the support as $\hat F = S_0 - \hat S$. The theorem  below, provides bounds on $\sum\limits_{j \in \hat F} \beta_j^2$, the signal strength of the undetected components; here we assume that $\sum\limits_{j \in \hat F} \beta_j^2 = 0$  if $\hat F =\emptyset$.

The following function of $k$ characterizes the probability of failure of the algorithm.
\eqq{\perr{k} = \PP(\mathcal{E}_{cond}) + 2(k+1)/p^a + 2k/p^{1+a},\quad \mbox{for}\,\, k \geq 1\label{prk},}
and $\perr{0} = 2/p^a$. Here, recall that $\mathcal{E}_{cond}$ is event that Conditions 1 or 2 fail. Notice that $\perr{k} \leq \perr{\bar k}$, since $k \leq \bar k$.

Regarding the choice of $a$, if $k$ is $O(\log p)$, then $a$ can be taken to be slightly larger than 0 for $\perr{k}$ to be small, assuming $p$ is large; however, if $k$ scales, for example, linearly with $p$, then $a$ needs to be taken to be larger than 1.
We now state our theorem.
\begin{thm} \label{thm:subgauss} 
  Let the threshold $\tau$ be as in (\ref{tauformintro}). 
Further, let $n$ be of the form \eqq{n = \xi \bar k\tau^2\label{nsuffcondsubgaus},}
 with $\xi$ as in (\ref{xidef}).

 Then, if $k \geq 1$, the following condition holds,
except on a set with probability $\perr{k}$:
 \eqq{\hat S \subseteq S_0 \quad\mbox{and}\quad \displaystyle\sum\limits_{j \in \hat F} \beta_j^2 \leq \alpha|\hat F|.\label{eqthmsubgausscond}}
In particular, if $\beta_{min}^2 > \alpha$ then $\hat S = S_0$, that is the support is recovered exactly, with probability at least $1 - \perr{k}$.

If $k = 0$, $\hat S = \emptyset$ with probability at least $1 - \perr{0}$.
\end{thm}

Notice that $\alpha$ controls accuracy to which the support is estimated. Assuming $\hat F$ is non-empty, another way of stating the theorem is that the average signal strength of the undetected components, that is $\|\beta_{\hat F}\|^2/|\hat F|$, is at most $\alpha$. It may seem desirable to make $\alpha$ as small as possible, however, doing so increases the value of $n$ in (\ref{nsuffcondsubgaus}), since $n$ is inversely related to $\alpha$ through $\xi(\alpha,\, \delta)$. Further, if $\alpha$ is taken to be  less than $\beta_{min}^2$, then the above theorem guarantees exact recovery of the support. Correspondingly, from \eqref{nsuffcondsubgaus} and \eqref{xidef}, one sees that if
$$n = \max\left\{ b_1 \bar k,\,  \frac{b_2}{\beta_{min}^2}\right\}\log p,$$
for some $b_1,\, b_2 >0$, then the support can recovered exactly with high probability.

The following corollary, which is a consequence of Theorem \ref{thm:gauss}, shows that if $n = \Omega(\bar k \log p)$, one can reliably detect the indices with large coefficient values, while ensuring that there are no false discoveries. Further, if all the non zero components are above the noise level (up to a constant factor), one can estimate the support exactly with the same number of observations.
\begin{cor}\label{cor:subgaussnoiselevel}
Define $\bar\xi = 32r_2^2(1 + a)$ and $r = 2r_2\sqrt{1 + a}$. Let \eqq{n \geq \bar\xi\, \bar k\log p.\label{nomega}} Then, if $k \geq 1$, with probability at least $1 - \perr{k}$, the estimate $\hat S$ is contained in $S_0$ and further,
$$\left\{j : |\beta_j| > r\,\sigma\sqrt{k}\mu_n\right\} \subseteq \hat S.$$
Further, if $\beta_{min} > r\,\sigma\mu_n$, then algorithm can recover the entire support of $\beta$, that is $\hat S = S_0$, with probability at least $1  - \perr{k}$.

If $k = 0$, then $\hat S = \emptyset$ with probability at least $1 - \perr{0}$. Here $\perr{.}$ is as in (\ref{prk}).
\end{cor}

\subsection{More general results with Gaussian designs} \label{subsec:gauss}

For Gaussian ensembles, the methods used in the proof of Theorem \ref{thm:subgauss} can be extended to give more general results on support recovery. In particular, we relax the assumption that $X$ has i.i.d entries and assume that rows of the $X$ matrix are i.i.d $N_p(0, \Sigma)$. The noise vector is assumed to be independent of $X$, with entries i.i.d. $N(0,\sigma^2)$. As mentioned earlier, here we also address a more general type of variable selection question, where we are not interested in recovering all non-zero entries but only the ones that are large compared to the noise level. In particular, for a constant $\maxinf \geq 0$,  let $S$ be a set of size $k$ as in (\ref{sassump}), consisting of the indices corresponding to the larger elements (in magnitude) of $\beta$. 
Once again, we do not assume that $k$ is known, but only assume that we have an upper bound $\bar k$ on $k$, with $\bar k \geq 1$. Unlike before, we do not require that the coefficients outside of $S$ are zero, but only assume that that $\|\beta_{S^c}\|_1\leq \sigma\sprs\mu_n$, where $\sprs$ is allowed to scale at most linearly with $\bar k$, that is we assume that $\bar\sprs = \eta/\bar k$ is $O(1)$. 

Through a permutation of the columns one can, without loss of generality, write $\Sigma$ as
 $$\Sigma =\left[\begin{array}{c c}\Sigma_{SS} & \Sigma_{SS^c}\\
\Sigma_{S^cS} &\Sigma_{S^cS^c}\\\end{array}\right],$$
where for $\mathcal{A},\, \mathcal{A}' \subseteq J$, $\Sigma_{\mathcal{A}, \mathcal{A}'} = \mbox{Cov}(X_{1,\mathcal{A}}, X_{1,\mathcal{A}'})$ is the covariance matrix between terms in $\mathcal{A}$ and $\mathcal{A}'$. We denote the elements of the matrix as $\sigma_{ij}$, or $\Sigma_{ij}$, and use both notations interchangeably. Without loss, we assume that $\sigma_{jj} = 1$ for all $j$, since if this were not the case, we could always scale the coefficient vector to produce such a correlation matrix.

We make the following assumptions on the correlation matrix $\Sigma$, when $k \geq 1$. These are essentially population analogs of the sparse eigenvalue  and the irrepresentable conditions respectively.
\begin{enumerate}
\item There exists $s_{min},\, s_{max} >0$ so that,
\eqq{\lambda_{min}(\Sigma_{TT}) \geq s_{min} \quad \mbox{and}\quad \lambda_{\max}(\Sigma_{TT}) \leq s_{min}\label{eigenpop},}
uniformly for all subsets $T$, with $|T| = k$. Here $\lambda_{min}(A),\, \lambda_{max}(A)$ denotes the minimum and maximum eigenvalues respectively of a square matrix $A$.
\item  For some $\maxcorrlone \in [0,1)$, the following holds,
\eqq{\max_{j \in J - T} \|\Sigma_{TT}^{-1}\Sigma_{Tj}\|_{1} \leq \maxcorrlone\label{irrepresentablecond},}
uniformly for all subsets $T$ of size $k$. This is essentially the population analog of the irrepresentable condition (\ref{irrepresentablecondpop}).
\end{enumerate}

Additionally, for $k \geq 1$, we make the following assumption that imposes bounds on certain interactions between $\beta_{S^c}$ and the correlation matrix $\Sigma$. 
As stated below, they are not very intuitive. Lemma \ref{lemsimpsuffcond}, however, shows that under a simple condition, which controls the magnitude of correlations of the off diagonal elements of $\Sigma$, and along with (\ref{betascl1}),  one can show (\ref{eigenpop}) - (\ref{betalinfconstrain}) to hold.

Let $\Sigma_{S^c|S} = \Sigma_{S^c S^c} - \Sigma_{S^c S}\Sigma_{SS}^{-1}\Sigma_{SS^c}$, denote the variance of the conditional distribution of $X_{1,S^c}$ given $X_{1,S}$, where we recall that $S$ is the subset of indices comprising of the $k$ largest elements (in magnitude) of $\beta$. Let $\mu_n$ be as in (\ref{mundef}). We make the following additional assumption.
\begin{enumerate}
\item[3.] For constants $\maxcon,\, \tilde\maxcon \geq 0$, the following holds, \eqq{\|\Sigma_{SS}^{-1}\Sigma_{SS^c}\beta_{S^c}\|_\infty \leq \sigma\tilde\maxcon\,\mu_n \quad\mbox{and}\quad \|\Sigma_{S^c|S}\beta_{S^c}\|_\infty \leq \sigma\maxcon\,\mu_n\label{betalinfconstrain}.}
\end{enumerate}


Notice that condition (\ref{betalinfconstrain}) is not required when $\beta$ is exactly sparse, that is when it has $k$ non-zero entries, since in this case $\beta_{S^c}$ is identically equal to zero. In this case, assumptions (\ref{eigenpop}, \ref{irrepresentablecond}) for exactly sparse vectors are identical to the sufficient conditions for support recovery for the Lasso by \citet{wainwright2009sharp}.

As an example, for the standard gaussian design, condition (\ref{eigenpop}) is satisfied with $s_{min}=s_{max}=1$. Condition (\ref{irrepresentablecond}) is satisfied with $\maxcorrlone = 0$. Condition (\ref{betalinfconstrain}) reduces to requiring that $\max_{j \in S^c}|\beta_j| \leq \sigma\maxcon\,\mu_n$, which is satisfied with $\maxcon = \maxinf$.

For the case $k = 0$, instead of (\ref{eigenpop}) - (\ref{betalinfconstrain}), we only make the assumption,
\eqq{\|\Sigma\,\beta\|_\infty \leq \sigma\maxcon\,\mu_n\label{assumk0}.}
Notice that since in this case $S = \emptyset$ and $J = S^c$, alternatively, one may express the left side of the above as
$\|\Sigma_{S^c|S}\beta_{S^c}\|_\infty$.


It is well known, see for example \citet{cai2010orthogonal}, \citet{tropp2004greed}, that if the correlations between any two distinct columns are small, as given by the incoherence condition, it implies both the sparse eigenvalue condition (\ref{eigenpop}) as well as the irrepresentable condition (\ref{irrepresentablecond}). We use these results to give simple sufficient conditions for (\ref{eigenpop}) - (\ref{betalinfconstrain}), as well as (\ref{assumk0}) when $k = 0$, in the following lemma. For this, define the coherence parameter,
\eqq{\gamma \equiv \gamma(\Sigma) = \max_{1 \leq j \neq j' \leq p} |\Sigma_{jj'}|\label{eqincoherence}.}
Further, recall that $\bar \sprs = \sprs/\bar k$. Then we have the following.
\begin{lem}
  \label{lemsimpsuffcond} Let $S$, with $|S| = k$,  be as in (\ref{sassump}). Assume that the correlation matrix $\Sigma$ satisfies,
  \eqq{\gamma(\Sigma) \leq \maxcorr/(2\bar k), \quad \mbox{where} \quad 0\leq \maxcorr < 1 \label{eqincoherencecond}.}
  Further, assume that the coefficient vector $\beta$ satisfies, for some $\sprs \geq 0$,
  \eqq{\|\beta_{S^c}\|_1 \leq \sigma\sprs\mu_n\label{eqbetasimpsuffcond}.}
  Define:
  \algg{s_{min} = 1 - \maxcorr/2\quad \quad \quad s_{max} &= 1 +\maxcorr/2\quad \quad\quad\maxcorrlone =\maxcorr\label{derivsminmax}\\
  \tilde\maxcon = \maxcorr\bar\sprs\quad\quad &\maxcon = \maxinf + \maxcorr\bar\sprs,\label{derivkappa}
        }
  Then, conditions (\ref{eigenpop}) - (\ref{betalinfconstrain}) holds, for $k = 1,\, \ldots,\, \bar k$, with the above values of $s_{min},\, s_{max},\, \maxcorrlone,\, \maxcon$ and $\tilde\maxcon$.

  If $k = 0$, condition (\ref{assumk0}) holds with $\maxcon$ in (\ref{derivkappa}).
\end{lem}

The above lemma is proved in Appendix \ref{pf:lemsimpsuffcond}. Equation (\ref{eqincoherencecond}) controls the maximum correlation between distinct columns and can be regarded as the population analog of the incoherence condition (\ref{incoherence}). Condition (\ref{eqbetasimpsuffcond}) imposes that $\beta_{S^c}$  has $\ell_1$ norm that is $O(\sprs\mu_n)$, where as mentioned before, $\sprs$ is allowed to scale at most linearly with $\bar k$.

Henceforth, for convenience sake, assume that we have control over the incoherence parameter as in (\ref{eqincoherencecond}) and that $\beta$ satisfies (\ref{eqbetasimpsuffcond}). Further, the quantities
$s_{min},\, s_{max},\, \maxcorrlone,\, \maxcon$ and $\tilde\maxcon$ will be as in (\ref{derivsminmax}) and (\ref{derivkappa}).

Condition (\ref{eqbetasimpsuffcond}) is more appropriate than an $\ell_1$ constraint on the whole vector $\beta$ since it does not impose any constraint on the larger coefficient values. Since the $\beta_j$, for $j \in S^c$, has magnitude at most $\sigma\maxinf\mu_n$, which is of the same order as the noise level, it makes sense for any algorithm to only estimate $S$ accurately. In Theorem \ref{thm:gauss} below, we give sufficient conditions on $n$ so that one can reliably estimate $S$. We note that this goal is different from that required in \citet{zhang2008sparsity} for support recovery with approximately sparse $\beta$. There, the only constraint on $\beta$ was that  $\|\beta_{A_0}\|_1 = O(\sprs\mu_n)$, for some set $A_0$, with $|A_0^c| = k$, and where $\sprs$ is also allowed to grow at most linearly $k$. Since there was no constraint on the magnitude of  $\beta_j$, for $j \in A_0$, some these $\beta_j$'s may have magnitude as high as $O(k \mu_n)$. For this reason, it made no longer sense to estimate $A_0^c$ accurately. Their criterion for an estimate $\hat S$ to be good was that $|\hat S| = O(k)$ and that the least squares fit of $Y$ on the columns in $\hat S$ produced a good approximation to $X\beta$.

The quantities $\lambda_{min},\, \lambda_{max}$ and $\lambda$ are redefined here. These will now be expressed as functions $\maxinf,\, \maxcorr$ and $\sprs$ using the various quantities $s_{min},\, s_{max},\, \maxcorrlone,\, \tilde\maxcon$ and $\maxcon$  defined in (\ref{derivsminmax}) and (\ref{derivkappa}).

We will need that the quantity $h = \sqrt{k/n} + \mu_n$ to be strictly less than one. Below, we arrange $n > 2\bar k\log p$. Correspondingly, one sees that $h < 1$ if, for example, $\bar k \geq 5$ and $p \geq 8$. Let $h_\ell = (1 -h)^2$ and $h_u = (1 + h)^2$.
We define the values of $\lambda_{min},\, \lambda_{max}$ and $\lambda$ in the following manner:
 \eqq{\lambda_{min} = s_{min}h_\ell \quad \mbox{and} \quad\lambda_{max} = s_{max}h_u\label{lambdaminmaxgaus}.}
Further, \eqq{\lambda = (1 + s_{max}^2\tilde\maxcon^2 + \maxcon\bar{\sprs})\left(1 + \bar k^{-1/2}\right)^2\label{lambdagauss}.}
Let $r_1$ be as in (\ref{equationr1r2}), now replaced with the above values of $\lambda_{min},\, \lambda_{max},\, \lambda$. The quantity $r_2$ is now given by,
\eqq{r_2 = \left[(1 - \maxcorrlone)\left(\tilde\maxcon+ \sqrt{\frac{1 + \maxcon\bar{\sprs}}{\lambda_{min}}}\right) + \sqrt{r_1}\right]\label{equationr2gaus}.}
Notice that for the i.i.d Gaussian ensemble and when $\beta$ is $k$-sparse, the quantities $\maxcorrlone,\, \tilde\maxcon,\, \maxcon$ and $\bar{\sprs}$ can be taken as zero. Correspondingly, $r_2$ has the same form as that in (\ref{equationr1r2}).

Further, let $\xi = \xi(\alpha,\, \delta)$ be as in (\ref{xidef}), with $r_1$ and $r_2$ appearing in its definition replaced with the values of these quantities defined above. The quantity $\perrp{k}$, for $k \geq 1$, which controls the probability of failure of the algorithm, is defined as,
\eqq{\perrp{k} = 4/p\, + \frac{\sqrt{2/\pi}}{\tau}\left[(k+1)/p^a + k/p^{1+a}\right]\label{prkgauss}.}
We define $\perrp{0} = 1/p + \sqrt{(2/\pi)}/(\tau p^a)$.
The threshold will now be denoted as $\tau_1$. It will be greater than $\tau$ by a factor $\rho \geq 1$. This factor is strictly greater than one if $\beta$ is not $\ell_0$-sparse or if  $\gamma(\Sigma)$ is non-zero.
 We are now in a position to state our main theorem. 

\begin{thm} \label{thm:gauss} Let the assumptions of Lemma \ref{lemsimpsuffcond} hold. Set the
threshold as $ \tau_1 = \inftau\, \tau$, where $\tau$ as in (\ref{tauformintro}),
and
\eqq{\inftau = \frac{\maxcon\left(1 + \bar k^{-1/2}\right) + 1}{1 - \maxcorrlone}\label{tildec}.}
Further, let \eqq{n = \xi\,\bar k\tau_1^2\label{nsuffcond1}.}
Then, if $k\geq 1$ the following holds with probability at least $1 - \perrp{k}$:
\eqq{\hat S \subseteq S\quad \mbox{and} \quad \sum_{j \in \hat F}\beta_j^2 \leq \alpha|\hat F|
\label{eqthmgausscond},}
where $\hat F = S - \hat S$. In particular, if $\beta_j^2 > \alpha$, for all $j \in S$, then $\hat S = S$ with probability at least $1 - \perrp{k}$.

If $k =0$, one has that $\hat S = \emptyset$ with probability at least $1 - \perrp{0}$.
\end{thm}


 Before stating the analog of Corollary \ref{cor:subgaussnoiselevel}, as an aside, we give implications of the above theorem for exact recovery of support for $k$-sparse vectors and i.i.d designs for large $n,\, p$ and $k$. This will help in understanding the results of Theorem \ref{thm:gauss} better.

In \cite{wainwright2009sharp} it was shown that for $k$-sparse vectors and i.i.d Gaussian designs that there is a sharp threshold, namely $n \asymp 2k\log p$, for exact recovery of the support as $n,\, p,\, k$, as well as $k\beta_{min}^2/\sigma^2$, tends to infinity. This was also proved for the OMP in \cite{fletcherorthogonal}, under an additional condition on rate of increase of the signal-to-noise ratio ($\|\beta\|^2/\sigma^2$). We can get similar results using our method by recalling that for i.i.d Gaussian designs and exact sparse vectors, $s_{min} = s_{max} = 1$ and $\maxcorrlone,\, \maxcon,\, \tilde \maxcon$ and $\sprs$ are all zero. Further, take $\bar k =k$. Correspondingly, since $h$ goes to 0, the quantities $\lambda_{min},\, \lambda_{max}$ and $\lambda$ in (\ref{lambdaminmaxgaus}, \ref{lambdagauss}) tend to 1 as $n,\, p$ and $k$ become large. This implies that $r_1$ tends to one and $r_2$ (\ref{equationr2gaus}) tends to 2. Further, as $k\beta_{min}^2/\sigma^2$ tends to infinity, one may also allow $k\alpha/\sigma^2$ tend to infinity, while keeping $\alpha < \beta_{min}$. From Theorem \ref{thm:gauss}, this will ensure that the support will be recovered exactly. Next, let's evaluate the quantity $\xi$ (\ref{xidef}) appearing in the expression for $n$. As $k\alpha/\sigma^2$ tends to infinity, one sees that the first term in the maximum in (\ref{xidef}) is the active one and hence $\xi$ tends to $(1 +\delta)$ (using $r_1$ tends to 1). One may also appropriately choose $\delta$ to tend to zero, making $\xi$ tend to 1. Accordingly, from (\ref{nsuffcond1}), one sees that if $n \approx 2(1 + a) k\log p$, for large $k,\, p$, one can recover the support exactly, with probability at least $1 - \perrp{k}$. When $\beta$ is extremely sparse, for example, when $k = O(\log p)$, then it is possible to arrange for $a$ to decrease to 0, while making $\perrp{k}$ also to 0. In this case, one gets the threshold $n \approx 2 k\log p$ for exact recovery. However, in the regime where $k$ is not negligible compared to $p$ (for example, when $k/p$ is constant), then our results only allow for $a$ to tend to 1 (from above), so as to ensure $\perrp{k}$ goes to zero. In this case our results are slightly inferior, requiring $n \approx 4k\log p$ for exact recovery. We remark in Section \ref{sec:conclusion} on how the results in \cite{fletcherorthogonal} may be carried over to the general case analyzed here.

  We now state the analog of Corollary \ref{cor:subgaussnoiselevel}. The goal now is not to recover the non-zero entries, but only those that are large compared to the noise level, which is a subset of $S$. 
We have the following.

\begin{cor}\label{cor:gaussnoiselevel} Let the assumptions of Lemma \ref{lemsimpsuffcond} hold and set the threshold to be $\tau_1$ as in Theorem \ref{thm:gauss}. Define $\bar\xi = 32(r_2\inftau)^2(1 + a)$ and $r = 2r_2\rho\sqrt{1 + a}$, where $r_2$ as in (\ref{equationr2gaus}). Let \eqq{n \geq \bar\xi\, \bar k\log p.\label{nomega}} Then, if $k\geq 1$, with probability at least $1 - \perrp{k}$, the estimate $\hat S$ is contained in $S$ and,
\eqq{\left\{j : |\beta_j| > r\,\sigma\sqrt{k}\mu_n\right\} \subseteq \hat S.\label{gausssqrtk}}
Further, if $|\beta_{j}| > r\,\sigma\mu_n$, for all $j \in S$, one has $\hat S = S$ with probability at least $1 - \perrp{k}$.

If $k = 0$, then $\hat S$ is $\emptyset$ with probability at least $1 - \perrp{0}$.
\end{cor}

Corollary \ref{cor:gaussnoiselevel} gives strong performance guarantees for the OMP under an incoherence property on the correlation matrix and an $\ell_1$ constraint on the smaller coefficients. From (\ref{gausssqrtk}), one sees that the larger coefficients, that is, those with magnitude $\Omega(\sqrt{k}\mu_n)$, are contained in $\hat S$ with high probability. Better performance can be demonstrated when all $\beta_j$'s, for $j \in S$, have magnitude $\Omega(\mu_n)$. 
In this case, it is possible to recover  $S$, while ensuring that there are no false positives. This is in a sense ideal, since it is nearly what one would expect in the orthogonal design case discussed in the beginning of Section \ref{sec:results}. In this case, assuming $\hat S$ is as in (\ref{zjhats}), one sees that in order to prevent false positives, $t$ needs to be $\Omega(\mu_n)$. Thus $|\beta_j|$, for $j \in S$, also needs to be $\Omega(\mu_n)$, with a slightly larger constant, to ensure  $\hat S = S$. For example, if the $|\beta_j|$'s, for $j \in S$, is  at least $\tilde t = (\maxinf + 2\sqrt{1 + a})\sigma\mu_n$, then it is not hard to see that the probability $\hat S = S$ is at least $1 - 2/p^a$.  Of course, the factor of $r\sigma$ obtained here, is larger than the corresponding factor for the orthogonal case, since the $X$ matrix is  in general quite far from being orthogonal; indeed, it is singular when $p > n$.



As a consequence of the above, we state results demonstrating strong oracle inequalities for parameter estimation under the $\ell_2$-loss.

\subsubsection{Oracle inequalities under $\ell_2$-loss}\label{subsec:paramest}
Let $\hat \beta$ be the coefficient estimate obtained after running the algorithm. More explicitly,  $(\hat \beta_j : j \in \hat S)$ is simply the least squares estimate when $Y$ is regressed on $X_{\hat S}$ and $\hat \beta_j = 0$ for $j \in \hat S^c$. 

We assume that the correlation matrix $\Sigma$ satisfies (\ref{eqincoherencecond}), that is,
   \eqq{\gamma(\Sigma) \leq \maxcorr/(2\bar k) \label{eqincoherencecondparam},}
 where $0\leq \maxcorr < 1$.

For simplicity, we consider the case that $\beta$ satisfies (\ref{sassump}) with $\maxinf =1$, that is,
 \eqq{S = \{j : |\beta_j| > \sigma\mu_n\} \quad \mbox{and}\quad \|\beta_{S^c}\|_1 \leq \sigma\sprs\mu_n\label{eqbetasimpsuffcondp},}
 where $|S| = k$ and $\sprs$ is allowed to grow at most linearly with $\bar k$, that is $\bar\sprs = \sprs/\bar k$ is $O(1)$.  With $\maxinf = 1$, $S$ denotes the set of indices greater than the noise level.

For the above values of $\sprs,\, \maxcorr$ and with $\maxinf = 1$, evaluate the quantities $s_{min},\, s_{max}$ as well as $\tilde \maxcon,\, \maxcon$ and $\maxcorrlone$ using expressions (\ref{derivsminmax}) and (\ref{derivkappa}). Evaluate $r_2$ as in (\ref{equationr2gaus}),  where the quantities $\lambda,\, \lambda_{min},\, \lambda_{max}$ are calculated using equations (\ref{lambdaminmaxgaus}, \ref{lambdagauss}). Further, let $\bar \xi$ and $r$ be as in Corollary \ref{cor:gaussnoiselevel}. Then we have the following.


\begin{thm}\label{thm:paramestl1sparse} Let (\ref{eqincoherencecondparam}) and (\ref{eqbetasimpsuffcondp})  hold. For fixed such $\beta$, if $$n \geq \bar \xi\, \bar k\log p,$$ then the  following holds  with probability at least $1 - \perrp{k}$:
\eqq{\|\hat \beta - \beta\|^2 \leq  C\sum_{j=1}^p \min\left(\beta_j^2, \sigma^2 \mu_n^2\right)\label{betal1sparsebddimp},}
where $C = (4/9)r^2$.
\end{thm}

The above theorem is essentially the analog of similar results for the Lasso  \cite[Corollary 6.1]{zhang2009some} and Dantzig selector \cite[Theorem 1.2]{candes2007dantzig}.
Note, the latter assumes that $\beta$ is $k$-sparse. Our results are more general since we only assume that the $\ell_1$ norm of the smaller coefficients satisfies a certain bound.
We proceed to state the corollary of the result assuming $\beta$ is $k$-sparse.

For $k$-sparse $\beta$, we only assume that (\ref{eqincoherencecondparam}) holds. Take $\sprs = \bar k$, so that $\bar \sprs = 1$. Evaluate $r_2$ using this values of $\sprs$, and with $\maxinf =1$, and call it $r_2^*$, that is,
\eqq{r_2^* = \left[(1 - \maxcorr)\left(\maxcorr + \sqrt{\frac{2 + \maxcorr}{\lambda_{min}}}\right) + \sqrt{r_1}\right]\label{r2ksparse},}
where once again, the quantities $r_1$ and $\lambda_{min}$ as calculated using (\ref{equationr1r2}, \ref{lambdaminmaxgaus}) and equations (\ref{derivsminmax}) and (\ref{derivkappa}). Further, let $\xi^*$ have the same expression as $\bar \xi$, except it is evaluated using $r_2^*$ instead of $r_2$. Similarly, let $r^*= 2 r_2^*\rho \sqrt{1 +a}$. Then we have the following.

\begin{cor}\label{thm:paramestksparse} Let (\ref{eqincoherencecondparam}) hold and let $\beta$ be a fixed $k$-sparse vector, for some $k \geq 0$. If $$n \geq \xi^*\, \bar k\log p,$$ then for $C_1 = (4/9)(r^*)^2$, the  following holds except on a set with probability $\perrp{k}$:
\eqq{\|\hat \beta - \beta\|^2 \leq  C_1\sum_{j=1}^p \min\left(\beta_j^2, \sigma^2 \mu_n^2\right).\label{betaksparsebddimp}}
\end{cor}

We now proceed to give proofs of our main results. The proofs employs techniques developed in  \citet{zhang2009consistency} and \citet{tropp2007signal}.

\section{Proof of results in Subsection \ref{subsec:subgaus}}\label{sec:proofsubgauss}

\begin{proof}[Proof of Theorem \ref{thm:subgauss}]
The following statistics will be useful in our analysis. Denote,
\eqq{\Zcal_{i} = \max_{j \in S_0}|\Zcal_{ij}|\quad\mbox{and}\quad \tilde{\Zcal}_i  = \max_{j \in S_0^c}|\Zcal_{ij}| \label{zcalmaxes}}
Notice if $\Zcal_i > \tau$ and $\Zcal_i > \tilde{\Zcal}_i$, then the index detected in step $i$, that is $a(i)$, belongs to $S$.

We first prove for the case $k \geq 1$. Let $\me$ be the event that statement (\ref{eqthmsubgausscond}) in Theorem \ref{thm:subgauss} does not hold. We want to show that the probability of $\me$ is small.
 There are two types of errors that we wish to control. Let $\me_1$ be the event that $\hat S$ in not contained in $S_0$. Further, let $\me_2$ be the event that $\hat S$ is contained is $S_0$, however $\sum_{j \in \hat F}\beta_j^2 > \alpha|\hat F|$. Clearly, $\me = \me_1 \cup \me_2$.

We use an argument similar to that used in \citet{tropp2007signal}. We initially pretend that $X = X_{S_0}$ and that the coefficient vector $\beta$ is shortened to a $k\times 1$ vector $\beta_{S_0}$ with all non-zero entries. Notice that $Y = X_{S_0}\beta_{S_0} + \epsilon$. For a given threshold $\tau$, we run the algorithm on this truncated problem. Let $m \leq k$ be the number of steps and let $\tilde R_1,\, \tilde R_2,\ldots,\, \tilde R_m$ be the associated residuals after each step. Also, denote as $\tilde R_0$ the vector $Y$. Notice that $m,\,\tilde R_0,\,\tilde R_1,\, \ldots,\, \tilde R_m$ are functions of $A = [X_{S_0} : \epsilon]$.

Let $\me_u$ be the event that statement (\ref{eqthmsubgausscond}) does not hold for the truncated problem. More explicitly, taking $\hat S_1 =\hat S(Y, X_{S_0}, \tau)$ and $\hat F_1 = S_0 - \hat S_1$, it is the event that $\|\beta_{\hat F_1}\|^2 > \alpha|\hat F_1|$.

Denote $T_i = \max_{j \in S_0}\left|X_j^{\trn}\tilde R_{i-1}/\|\tilde R_{i-1} \|\right|$ and $\tilde T_i = \max_{j \in S_0^c}\left|X_j^{\trn}\tilde R_{i-1}/\|\tilde R_{i-1} \|\right|$, for $i = 1,\,\ldots,\, m+1$. Notice that the statistics $T_i,\, \tilde T_i$ are similar to $\Zcal_i,\, \tilde{\Zcal}_i$, the only difference being that the residuals involved in the former arise from running the algorithm on the truncated problem, whereas in the latter they arise from consideration of the original problem. Further, let $\me_f$ be the event
$$\mathcal{E}_f =\left\{\mbox{$\tilde T_i > \tau,\,\tilde T_i \geq T_i$ for some $i\leq m + 1$}\right\}.$$
We now show that $\me \subseteq \me_u \cup \me_f$. To see this, write $\me$ as a disjoint union
$\me_1\cup\tilde{\me}_2$, where  $\tilde{\me}_2 = \me_2 \cap \me_1^c$. Let's first consider the case that $\tilde{\me}_2$ occurs. Clearly this means that $\me_u$ has occurred if the algorithm were run on the truncated problem for the given $A$.

 Next, consider the case that $\me_1$ occurs. Let $R_0,\, R_1\ldots \mbox{etc.}$ be the residuals for the original problem (\ref{eqone}), for the given realization of $[X : \epsilon]$. Let $i^*$ be the step for which the false alarm occurs for the first time. Clearly, $i^* \leq m+1$, since otherwise it would mean that the truncated problem (with $X = X_{S_0}$) ran  for more than $m$ steps.
Also, we must have $\{\Zcal_i > \tau,\,\Zcal_i > \tilde{\Zcal}_i \}$ occur for $1\leq i \leq i^* -1$ and $\{\tilde{\Zcal}_{i^*} > \tau, \tilde{\Zcal}_{i^*} \geq \Zcal_{i^*}\}$ occur. Correspondingly, one sees that $R_0 = \tilde R_0,\,\ldots,\, R_{i^* - 1} =\tilde R_{i^* - 1}$, which implies that
$T_{i^*} = \Zcal_{i^*}$ and $\tilde{T}_{i^*} = \tilde{\Zcal}_{i^*}$. Consequently, as $\{\tilde{T}_{i^*} > \tau,\, \tilde{T}_{i^*} \geq T_{i^*}\}$ occurs, $\me_f$ occurs. Hence, $\mathcal{E} \subseteq \mathcal{E}_u \cup \mathcal{E}_{f}$ which gives, $$\PP(\mathcal{E}) \leq \PP(\mathcal{E}_u) + \PP(\mathcal{E}_{f}).$$ 
Consequently, all we are left with is to bound the probabilities of $\me_f$ and $\me_u$.

We first bound the probability of $\me_f$. For this, notice that $\me_f \subseteq \me_f'$, where $\me_f' = \{\max_{1 \leq i\leq m + 1 }\tilde T_i > \tau\}$. Since  $X_{S_0^c}$ is independent of $A = [X_{S_0} : \epsilon]$, one has that $X_{S_0^c}$ is independent of $\tilde R_1,\ldots,\tilde R_m$. Correspondingly, from Lemma \ref{subgaus} (a), conditional on $A$, we have that $X_j^{\trn}\tilde R_i/\|\tilde R_i\|$ is sub-gaussian with mean 0 and scale 1, for $j \in S_0^c$ and $1\leq i \leq m  + 1$. Consequently, using standard results on the maximum of sub-Gaussian random variables (Lemma \ref{subgaus} (b)), if $\tau$ be as in (\ref{tauformintro}),  one gets that  $\PP(\mathcal{E}_f | A) \leq 2(m + 1)/p^a$, using $|S_0^c| \leq p$. Since  $ m \leq k$, this probability is bounded by $2(k +1)/p^a$, which implies $\PP(\mathcal{E}_f) \leq  2(k  + 1)/p^a$. 

Next, we bound the probability of $\me_u$. For this, consider a linear model of the form,
\eqq{U = H\varphi + w\label{truncmodel},} where $H$ is an $n \times k$ matrix satisfying, $w$ an $n \times 1$ vector and $\varphi$ a $k \times 1$ dimensional coefficient vector. After running the OMP on this model (with $Y = U,\, X = H$ and threshold $\tau_0$), let $\hat S_2 = \hat S(U,\, H,\, \tau_0)$ be the estimate of the support. Further, let $\hat \varphi$ be the coefficient estimate obtained, that is, $(\hat \varphi_j : j \in \hat S_2)$ is the least squares estimate when $U$ is regressed on $H_{\hat S_2}$ and $\hat \varphi_j = 0$ for $j$ not in $\hat S_2$. We use the following Lemma, the proof of which is similar to the analysis in \citet{zhang2009consistency}.

\begin{lem} \label{zhnglem} For the model (\ref{truncmodel}), let the following hold.
\begin{enumerate}[(i)]
   \item Condition 1 holds for $H$, that is the eigenvalues of $H^{\trn}H/n$ are between $\lambda_{min}$ and $\lambda_{max}$.
  \item Condition 2 holds for $w$, that is $\|w\|^2 \leq n\sigma^2\lambda$, for some $\lambda >0$.
   \item $\|\hat \varphi_{ls} - \varphi\|_{\infty} \leq \sigma c_0\tau_0/\sqrt{n}$, for some constant $c_0 >0$, where $\hat \varphi_{ls}$ is the coefficient vector of the least square fit of $U$ on $H$.
\end{enumerate}
  Under the above, if  the OMP is run with $Y= U$, $X = H$ and threshold $\tau_0$, when the algorithm stops we must have the following,
  \begin{enumerate}[(a)]
  \item
  \eqq{\left(1 - \tau_0\sqrt{r_1 k/n}\right)\|\varphi_{\hat F_2}\| \leq \trt\sigma\tau_0\sqrt{\frac{|\hat F_2|}{n}}\label{stopineq},
  }
  where  $\hat F_2 = \{1,\ldots, k \} - \hat S_2$, denotes the indices not detected after running the algorithm. Further, $r_1$ has the same form as (\ref{equationr1r2}), replaced with the above values of $\lambda_{min},\, \lambda_{max}$ and $\lambda$. Also, $\trt = c_0 + \sqrt{r_1}.$
  \item
      \eqq{\|\hat \varphi - \varphi\| \leq \frac{\trt\sigma\tau_0\sqrt{k/n}}{1 - \tau_0\sqrt{r_1 k/n}}.}
  \end{enumerate}
\end{lem}

The above lemma is proved in Appendix \ref{pf:zhnglem}. We only require the conclusions in part (a) of the lemma for the time being. Part (b) will be required of Subsection \ref{subsec:paramest} to get bounds on $\ell_2$-error of the coefficient estimate.

Now apply Lemma \ref{zhnglem} to the truncated problem, that is, with $H = X_{S_0},\, \varphi = \beta_{S_0}$, $U = Y$ and $\tau_0 = \tau$.  Notice that in this case $\hat F_2 = \hat F_1$ and  $\hat S_2 = \hat S_1$. We know that requirements (i) and (ii) of the Lemma \ref{zhnglem} hold, except on a set $\me_{cond}$. The following lemma shows that (iii) holds with high probability.

\begin{lem}\label{maxbdd} Let $\hat \beta_{ls}$ be the least squares fit when $Y$ is regressed on $X_{S_0}$. Further, let $$\me_{ls} = \{\|\hat \beta_{ls} - \beta_{S_0}\|_\infty > \sigma c_0\tau/\sqrt{n}\},$$ where $c_0 = 1/\sqrt{\lambda_{min}}$. Then $\PP\left(\me_{ls}\cap\me_{cond}^c\right) \leq 2k/p^{1+a}$.
\end{lem}

The above lemma is proved after this proof. Using the above lemma, all requirements of Lemma \ref{zhnglem} hold, except on a set
$\tilde{\mathcal{E}}_u =\me_{cond}\cup\me_{ls}$, the probability of which is bounded by $\PP(\me_{cond}) + 2k/p^{1+a}$. We now show that $\mathcal{E}_u \subseteq
\tilde{\mathcal{E}}_u$. We do this by showing $\tilde{\me}_u^c\subseteq \me_u^c$. To see this, notice that on $\tilde{\mathcal{E}}_u^c$, one has
\eqq{\left(1 - \tau\sqrt{r_1 k/n}\right)\|\beta_{\hat F_1}\| \leq \trt\sigma\tau\sqrt{\frac{|\hat F_1|}{n}}\label{uiupp}.}
from \eqref{stopineq}. Assume that $\hat F_1$ is non-empty, since otherwise the claim is trivially true.
 Notice that since $n \geq (1 + \delta) r_1 \bar k  \tau^2$ from \eqref{nsuffcondsubgaus}, one has $\tau\left(\bar k r_1/n\right)^{1/2} \leq 1/\sqrt{1 + \delta}$. Now, since $k \leq \bar k$, the left side of (\ref{uiupp}) is non-negative. Thus, (\ref{uiupp})
can be reexpressed as,
$$\|\beta_{\hat F_1}\|^2 \leq (\sigma^2r_2^2f(\delta)\tau^2/n)|\hat F_1|,$$
which follows from noticing that $r_2 = \trt$, where $r_2$ is as in \eqref{equationr1r2}.
Now, since $n \geq \sigma^2r_2^2 f(\delta)\tau^2/\alpha$, the left side of the above is at most $\alpha|\hat F_1|$. Thus, $\sum_{j \in \hat F_1} \beta_j^2 \leq \alpha |\hat F_1|$
on $\tilde{\me}_u^c$, which implies that $\me_u \subseteq \tilde{\me}_u$.  Consequently, $\PP(\me_u) \leq \PP(\me_{cond}) + 2k/p^{1 + a}$. Accordingly, since $\PP(\me) \leq \PP(\me_u) + \PP(\me_f)$,  one has $\PP(\me) \leq \PP(\me_{cond}) + 2k/p^{1 + a} + 2(k + 1)/p^a$, which is equal to $\perr{k}$. This completes the proof for the case $k \geq 1$.

For the case $k = 0$, we just need to show that the algorithm stops after the first step, in which case $\hat S = \emptyset$. This is immediately seen by noticing that for $k=0$, one has that  $\Zcal_{1j}$, for $j \in J$, are sub-gaussian with mean 0 and scale 1. Correspondingly, from Lemma \ref{subgaus}(b), the event $\{\max_{j \in J} |\Zcal_{1j}| > \tau\}$ has probability at most $\perr{0} = 2/p^a$.
\end{proof}

\begin{proof}[Proof of Lemma \ref{maxbdd}]
Firstly, note that $\hat \beta_{ls} - \beta_{S_0}$ can be expressed as $Z = (X_{S_0}^{\trn}X_{S_0})^{-1}X_{S_0}^{\trn}\,\epsilon$. Let $Z = (Z_j : j = 1,\ldots, k)$. Now, conditioned on $X_{S_0}$, each $Z_j$ is sub-gaussian with mean 0 and scale $\sigma_j = \sigma\sqrt{e_j^{\trn}(X_{S_0}^{\trn}X_{S_0})^{-1}e_j}$. Here, $e_j$ is the $j$ th column of the size $k$ identity matrix. Correspondingly, from Lemma \ref{subgaus}(b), one gets $\max_j |Z_j|$ is less than $(\max_j \sigma_j)\tau$, except on a set with probability $2k/p^{1+a}$. Finally, observe that on $\me_{cond}^c$, one has $e_j^{\trn}(X_{S_0}^{\trn}X_{S_0})^{-1}e_j \leq 1/(n\lambda_{min})$, since the maximum eigenvalue of $(X_{S_0}^{\trn}X_{S_0}/n)^{-1}$ is at most $1/\lambda_{min}$. Thus, $\max_j \sigma_j\tau$ is at most $\sigma c_o\tau/\sqrt{n}$, with $c_0 = 1/\sqrt{\lambda_{min}}$.
\end{proof}

\begin{proof}[Proof of Corollary \ref{cor:subgaussnoiselevel}] Take $\alpha(\delta) = \sigma^2/[(1 + \delta)\bar k]$. Further, let $\xi(\delta) = \xi(\alpha(\delta),\, \delta)$, which, using $r_2^2 \geq r_1$ and $f(\delta) \geq 1$, can be written as,
\eqq{\xi(\delta) = (1 + \delta)f(\delta)r_2^2\label{xidel}.}
The function $(1 + \delta)f(\delta)$, for $\delta >0$, has its minimum at $\delta^* = 3$. Further, it is increasing and goes to infinity as $\delta$ tends to infinity. Now, using $\xi(\delta^*) = 16r_2^2$, notice that $\xi(\delta^*) \bar k\tau^2 =
\bar\xi\, \bar k\log p$. Correspondingly, since $n \geq \bar\xi\, \bar k \log p$, one gets that
\eqq{n = \xi(\delta) \bar k\tau^2\label{naltsubgauss},}
for some $\delta \geq \delta^*$. Consequently, from Theorem \ref{thm:subgauss}, one has,
 \eqq{\hat S \subseteq S_0 \quad\mbox{and}\quad \displaystyle\sum\limits_{j \in \hat F} \beta_j^2 \leq \alpha(\delta)|\hat F|\label{eqafterrunnoisesubgaus},}
with probability at least $1 - \perr{k}$.
Use $f(\delta) \leq f(\delta^*) = 4$, to get from (\ref{naltsubgauss}) that $n \leq (1 + \delta)r \bar k\tau^2$. Correspondingly, $\alpha(\delta)$ is at most $r^2\sigma^2\mu_n^2$. Consequently, any $j$, with $|\beta_j| > r\sigma\sqrt{k}\mu_n$ cannot be in $\hat F$ since it would contradict the inequality in (\ref{eqafterrunnoisesubgaus}).
Further, if $\beta_{min} > r\sigma\mu_n$, the inequality in (\ref{eqafterrunnoisesubgaus}) cannot hold if $\hat F$ is non-empty. In this case the algorithm recovers the entire support.
\end{proof}

\section{Proof of results in Subsection \ref{subsec:gauss}} \label{sec:proofguass}

\begin{proof}[Proof of Theorem \ref{thm:gauss}]
Once again, we first prove for the case $k \geq 1$. As before, we are interested in bounding the probability of $\me$, where $\me = \me_1\cup\me_2$. Here $\me_1$ is the event that $\hat S$ is not contained in $S = S(\beta)$. Also, $\me_2$ is the event $\hat S \subseteq S$ and $\|\beta_{\hat F}\|^2 > \alpha|\hat F|$, where, here $\hat F = S - \hat S$ and $\hat S = \hat S(Y, X, \tau_1)$. Write $Y$ as $Y = X_S \beta_S + \tilde \epsilon$, where $\tilde \epsilon = X_{S^c}\beta_{S^c} + \epsilon$.
Analogous to before, we initially pretend that $X = X_S$ and $\beta = \beta_S$ and run the algorithm on the truncated problem to get residuals $\tilde R_0,\, \tilde R_1,\, \tilde R_2,\ldots,\, \tilde R_m$. These residuals are functions of $A = [X_S : \tilde\epsilon]$.
Further, as before, let $\me_u$ be the event that statement (\ref{eqthmgausscond}) is not met for this truncated problem.  With $\hat S_1 = \hat S(Y, X_S, \tau_1)$ and $\hat F_1 = S - \hat S_1$, it is the event that
$\|\beta_{\hat F_1}\|^2 > \alpha|\hat F_1|$. Similarly, we define $T_i,\, \tilde T_i$ as before, now with the maximum taken over $S$ instead of $S_0$. Further, define the event $\me_f$ analogous to before, with $\tau$ replaced by $\tau_1$. Using the same reasoning as in Theorem \ref{thm:subgauss}, one has $\me \subseteq  \me_u \cup \me_f$. We first proceed to bound the probability of $\me_f$. Notice that unlike previously, the $X_j$'s, for $j \in S^c$, are not independent of the $\tilde R_i$'s.  This makes bounding the probability of $\me_f$ more involved.

The following lemma will be useful, both in bounding $\PP(\me_f)$ as well as $\PP(\me_u)$. We denote as $\hat \beta_{ls}$ the least square estimate when $Y$ is regressed on $X_S$.

\begin{lem}\label{anallem} Parts (i)-(iii) of this lemma demonstrate that requirements (i)-(iii) of Lemma \ref{zhnglem} are satisfied with high probability.
\begin{enumerate}[(i)]
  \item   With $\lambda_{min}, \, \lambda_{max}$ as in (\ref{lambdaminmaxgaus}), the following holds with probability at least
 1 - 2/p:
\eqq{\lambda_{min}\|v\|^2 \leq \|X_S v\|^2/n \leq \lambda_{max}\|v\|^2 \quad\mbox{for all}\quad v \in \Rbb^k. \label{isocov}}
\item Let $\lambda$ be as in (\ref{lambdagauss}). Then
$\|\tilde \epsilon\|^2/(n \sigma^2) \leq \lambda$,
with probability at least $1 - 1/p$.
\item  Let $\me_{ls} =\{\|\hat \beta_{ls} - \beta_S\|_\infty > \sigma c_0\tau_1/\sqrt{n}\}$, where
\eqq{c_0 = (1 - \maxcorrlone)\left[\tilde\maxcon + \sqrt{\frac{1 + \maxcon\bar \sprs}{\lambda_{min}}}\right]\label{cgauss}}
Then $\PP(\me_{cond}^c\cap \me_{ls}) \leq (\sqrt{2/\pi})k/(\tau p^{1+a})$, where  $\me_{cond}$, here, is the event that (i) or (ii) above fails. From (i) and (ii) it has probability at most $3/p$.
\end{enumerate}
\end{lem}

The above lemma is proved in Section \ref{pf:sec:proofgauss}.
As mentioned before,  the $X_j$'s, for $j \in S^c$, are not independent of the $\tilde R_i$'s. We get around this by finding the conditional distribution of each $X_j$ given $X_S$ and $\tilde \epsilon$. Correspondingly, each $X_j$ may be represented as a linear combination of columns in $A = [X_S : \tilde \epsilon]$ plus a noise vector, which we call $Z_j$. This noise term is independent of $A$ and hence $\tilde R_0,\, \tilde R_1,\ldots,\, \tilde R_m$.

 Let $a_j = \Sigma^{-1}_{SS}\Sigma_{Sj}$
 and
\eqq{b_j = \frac{e_j^{\trn}\Sigma_{S^c|S}\beta_{S^c}}{\sqrt{\ddd}} \label{bj},}
where $e_j$ is the $j$th column of the size $p-k$ identity matrix and
\eqq{\ddd = \sigma^2 + \beta_{S^c}^{\trn}\Sigma_{S^c|S}\beta_{S^c} \label{req}.}
 The following lemma characterizes the conditional distribution of $X_j$ given $A$.
\begin{lem}
  \label{lemconddist} Let $a_j,\, b_j$, for $j \in S^c$, be as above. Then we have the following:
  \begin{enumerate}[(i)]
  \item  The distribution of $X_j$, for $j \in S^c$, may be represented as
  \eqq{ X_j \stackrel{\mathcal{D}}{=}   X_S\, a_j + b_j W  + Z_j \label{distxj}}
  where $W \sim N(0, I_n)$ and is independent of $X_S$. Further,
  $Z_j$  is independent of $[X_S : \tilde \epsilon]$ and follows $N(0, \tilde \sigma_{jj} I_n)$, with $\tilde \sigma_{jj} \leq \sigma_{jj} = 1$.
  \item Define, for $j\in S^c$ and $i = 1,\, \ldots,\, m+1$,
   \eqq{V_{ji} =  b_j W ^{\trn}\frac{\tilde R_{i-1}}{\|\tilde R_{i-1}\|} + E_{ji},\label{eq:vji}}
   where $E_{ji} = Z_j^{\trn}\tilde R_{i-1}/\|\tilde R_{i-1}\|$.
   Let,
  \eqq{\tilde{\me}_f = \left\{ \max_{1 \leq i \leq m+1,\, j \in S^c}|V_{ji}| > (1 - \maxcorrlone)\tau_1\right\} \label{meftilde}.}
  Then $\PP(\tilde{\me}_f) \leq 1/p + (\sqrt{2/\pi})(k+1)/(\tau p^a)$.
  \end{enumerate}
\end{lem}

The above lemma is proved in Section \ref{pf:sec:proofgauss}. We now show that $\me_f \subseteq \tilde{\me}_f$. To see this, notice that on $\tilde{\me}_f^c$ one has,
\algg{\tilde T_i &\leq (\max_{j \in S^c}\|a_j\|_1) T_i + (1 - \maxcorrlone)\tau_1\nonumber\\
                 &\leq \maxcorrlone T_i + (1 - \maxcorrlone)\tau_1\label{eq:maxreason},}
 for $i = 1,\,\ldots,\, m+1$.
Here, the first inequality follows from using (\ref{distxj}) and $\left|a_j^{\trn}X_S^{\trn}\tilde R_{i-1}/\|\tilde R_{i-1}\|\right| \leq \|a_j\|_1 T_i$, along with the fact that $|V_{ji}|$ is bounded by $(1 -\maxcorrlone)\tau_1$ on $\tilde{\me}_f^c$. The second inequality follows from (\ref{irrepresentablecond}). We now show that $$\me' = \left\{\tilde T_i \leq \maxcorrlone T_i + (1 - \maxcorrlone)\tau_1\,\,\mbox{ for each $\,\,i\leq m+1$}\right\}$$ implies $\me_f^c$.  To see this, for each $i$, consider two cases, viz. $T_i > \tau_1$ and $T_i \leq \tau_1$. From \eqref{eq:maxreason}, in the first case one has $\tilde T_i < T_i$,  and in the second case, one has $\tilde T_i \leq \tau_1$. Correspondingly, $\me'$ is contained in $$\{\tilde T_i < T_i\,\,\mbox{or}\,\, \tilde T_i \leq \tau_1\quad\mbox{for each $i\leq m+1$}\},$$
which is $\me_f^c$.
Consequently, $\me_f \subseteq \tilde{\me}_f$. Consequently, $\PP(\me_f) \leq 1/p + (\sqrt{2/\pi})(k+1)/(\tau p^a)$ from Lemma \ref{lemconddist}.

What remains to be seen is that the probability of the event $\mathcal{E}_u$ can be bounded as before. For this we apply Lemma \ref{zhnglem} once again. That conditions (i) - (iii),
required for application of Lemma \ref{zhnglem}, are satisfied with high probability is proved parts (i)-(iii) of Lemma \ref{anallem}. Consequently, as before, if $\tilde{\me}_u = \me_{cond}\cup\me_{ls}$, where the sets on the right side are as in Lemma \ref{anallem}, one gets that on $\tilde{\me}_u^c$,
\eqq{\left(1 - \tau_1\sqrt{r_1 k/n}\right)\|\beta_{\hat F_1}\| \leq \trt\sigma\tau_1\sqrt{\frac{|\hat F_1|}{n}}\label{uiuppgauss}.}
Here $\trt = c_0 + \sqrt{r_1}$, where $c_0$ as in (\ref{cgauss}). Notice that $\trt =  r_2$, where $r_2$ as in \eqref{equationr2gaus}. Now, once again use the fact that $n \geq (1 + \delta)r_1k\tau_1^2$ and $n \geq r_2^2f(\delta)\sigma^2\tau_1^2/\alpha$, to get that (\ref{uiuppgauss}) implies $\me_u^c$. Accordingly, $\PP(\me_u) \leq\PP(\tilde{\me}_u)$.
Consequently, one has,
\alge{\PP(\me) &\leq \PP(\me_u \cup \me_f)\\
&\leq\PP(\me_{cond}) + \PP(\me_{cond}^c\cap\me_{ls}) + \PP(\me_{f}),
}
which is at most $\perrp{k} = 4/p + (\sqrt{2/\pi}/\tau)\left[(k+1)/p^a + k/p^{1 +a}\right]$. This completes the proof for $k\geq 1$.

If $k =0$, we will show that the probability that $\max_{j \in J}|\Zcal_{1j}|$ exceeds $\tau_1$ is at most $\perrp{0}$. This would imply that the algorithm stops after one step and $\hat S$ is empty.
Notice that $S^c = J$ and hence $\tilde \epsilon = Y$. Consequently,
 $X_j \stackrel{\mathcal{D}}{=} \tilde b_j Y/\sigma_Y + Z_j$, where $Z_j \sim N(0, \tilde\sigma_j)$ is independent of $Y$, with $\tilde \sigma_j \leq 1$.  Also, $\tilde b_j = e_j^{\trn}\Sigma\beta/\sigma_Y$, where
 $\sigma_Y^2 = \mbox{Var}(Y_1) = \sigma^2 + \beta^{\trn}\Sigma\beta$. Correspondingly,
 \eqq{\Zcal_{1j} \stackrel{\mathcal{D}}{=} \tilde b_j \|Y\|/\sigma_Y + Z_j^{\trn} \frac{Y}{\|Y\|}\label{k0gauss}
 }
 Using $\sigma_Y \geq \sigma$, one has $\tilde b_j \leq \maxcon\mu_n$. Further, using $\|Y\|/\sigma_Y \leq (1 + \mu_n)$, with probability at least $1 - 1/p$ from Lemma \ref{lem:chitail},  one has that the first term in the right side of (\ref{k0gauss}) is at most $\nu_1\tau(1 + \bar k^{-1/2})$ with probability at least $1 -1/p$. Further $|Z_j^{\trn} Y/\|Y\||$, using the independence of $Z_j$ and $Y$, is less than $\tau$ for all $j$ with probability at least $1 - \sqrt{2/\pi}/(\tau p^a)$ (Lemma \ref{subgaus} (b)). Denoting, $\tau_2 = [\nu_1(1 + \bar k^{-1/2}) + 1]\tau$, one sees $\max_{j \in J}|\Zcal_{1j}| \leq \tau_2$, with probability at least $1 - \perrp{0}$. Notice that since $\tau_1 \geq \tau_2$, the event $\max_{j \in J}|\Zcal_{1j}| \leq \tau_1$ also has probability at least $1 - \perrp{0}$. This completes the proof.
 \end{proof}

\begin{proof}[Proof of Corollary \ref{cor:gaussnoiselevel}] The proof is exactly similar to that of Corollary \ref{cor:subgaussnoiselevel}. As before, taking $\alpha(\delta) = \sigma^2/[(1 + \delta)\bar k]$ and $\xi(\delta) = \xi(\alpha(\delta),\, \delta)$, we notice that
$\rho^2\xi(\delta^*) \bar k\tau^2 = \bar\xi\, \bar k\log p$, where $\delta^* = 3$. Correspondingly, if $n \geq \bar\xi\, \bar k\log p$, one has
$n = \rho^2\xi(\delta) \bar k\tau^2 $ for some $\delta \geq \delta^*$ and hence,
 $$\hat S \subseteq S \quad\mbox{and}\quad \displaystyle\sum\limits_{j \in \hat F} \beta_j^2 \leq \alpha(\delta)|\hat F|$$
 with probability at least $1 - \perrp{k}$, from Theorem \ref{thm:gauss}. Further, $\alpha(\delta)$ is at most $r^2\sigma^2\mu_n^2$, using the same reasoning as before.   The conclusions on recovering the large coefficients follow immediately from this.
 \end{proof}

\begin{proof}[Proof of Theorem \ref{thm:paramestl1sparse}] Notice that,
\eqq{\|\hat \beta - \beta\|^2 = \|\hat\beta_S - \beta_S\|^2 + \|\hat\beta_{S^c} - \beta_{S^c}\|^2.\label{paramone}}
We apply the result of Corollary \ref{cor:gaussnoiselevel}, to get that except on a set with probability $\perrp{k}$, one has $\hat S \subseteq S$. Correspondingly, the second term in (\ref{paramone}) is simply $\|\beta_{S^c}\|^2$, which is equal to $\sum_{\j \in S^c} \min\{\beta_j^2,\, \sigma^2\mu_n^2\}$.

Let's next concentrate on the first term in (\ref{paramone}). Notice that since $\hat S \subseteq S$, one has $\hat \beta_S$ is same as the coefficient estimate one would get if the OMP were run on the truncated problem. Correspondingly, using part (b) of Lemma \ref{zhnglem}, with $\tau_0 = \tau_1$ and $\trt = r_2$, one gets that
\eqq{\|\hat\beta_S - \beta_S\| \leq  \frac{r_2\sigma\tau_1\sqrt{k/n}}{1 - \tau_1\sqrt{r_1 k/n}},\label{paramtwo}}
with probability at least $1 - \perrp{k}$. Next, use the fact that $\tau_1\sqrt{k/n} \leq 1/(4r_2)$ using
$\bar \xi\, \bar k \log p = 16r_2^2 \bar k\tau_1^2$. Consequently, the denominator in the right side of (\ref{paramtwo}) is at least $1 - \sqrt{r_1}/4r_2$. The latter is at least $3/4$ using $r_2 \geq \sqrt{r_1}$. Thus,
\algg{\|\hat\beta_S - \beta_S\| &\leq  \frac{4r_2\rho\sqrt{1+a}}{3}\sigma\sqrt{k}\mu_n,\nonumber\\
                                &=  \sqrt{C}\sigma\sqrt{k}\mu_n \label{paramthree},
}
where $C = (4/9)r^2$. Correspondingly, from (\ref{paramone}) one gets that,
\alge{\|\hat \beta - \beta\|^2 &\leq  C\sigma^2 k\mu_n^2 + \sum_{\j \in S^c} \min\{\beta_j^2,\, \sigma^2\mu_n^2\}\\
&\leq C\sum_{\j =1 }^p \min\{\beta_j^2,\, \sigma^2\mu_n^2\},}
where the last inequality from using $\sigma^2 k\mu_n^2 = \sum_{j \in S}\min\{\beta_j^2,\, \sigma^2\mu_n^2\}$, since $S = \{j : |\beta_j| > \sigma\mu_n\}$.
\end{proof}

\begin{proof}[Proof  of Corollary \ref{thm:paramestksparse}] For $k$-sparse $\beta$, once again let $S = \{j : |\beta_j|> \sigma\mu_n\}$. Now $\|\beta_{S^c}\|_1 \leq \sprs\sigma\mu_n$, where $\sprs = \bar k$, since there are at most $\bar k$ non-zero entries outside of $S$, with magnitude at most $\sigma\mu_n$. Now apply Theorem \ref{thm:paramestl1sparse}, with $\sprs = \bar k$ (or $\bar \sprs = 1$) to get the desired result.
\end{proof}

\section{Proof of results from Section \ref{sec:proofguass}}\label{pf:sec:proofgauss}

The following simple lemma will prove useful in proving Lemma \ref{anallem}.

 \begin{lem}\label{concond} Let $\theta_n = \bar k^{1/2}\mu_n$. Conditions (\ref{eigenpop}) - (\ref{betalinfconstrain}) imply the following:

  (i) Let $\ddd$ be as in (\ref{req}). Then $\ddd  \leq \sigma^2(1+ \maxcon\bar \sprs\,\theta_n^2).$

  (ii) $\|\Sigma_{SS}g\|^2 \leq \sigma^2s_{max}^2\tilde\maxcon^2\theta_n^2,$ where $g = \Sigma_{SS}^{-1}\Sigma_{SS^c}\beta_{S^c}$.
 \end{lem}
 \textbf{Remark:} Since we take $n > 2\bar k\log p$,  we have $\theta_n \leq 1$. Accordingly, the above bound holds with $\theta_n$ replaced by 1.

\begin{proof}[Proof of Lemma \ref{concond}] We first prove part (i). Recall that $\ddd =  \sigma^2 + \beta_{S^c}^{\trn}\Sigma_{S^c|S}\beta_{S^c}$. Write $\beta_{S^c}^{\trn}\Sigma_{S^c|S}\beta_{S^c}$ as $\sum_{j \in S^c}\beta_j e_j^{\trn}\Sigma_{S^c|S}\beta_{S^c}$, which can be bounded by $(\|\Sigma_{S^c|S}\beta_{S^c}\|_\infty)\|\beta_{S^c}\|_1$, which is at most $\sigma\maxcon\bar\sprs\,\theta_n^2$ from (\ref{betalinfconstrain}) and (\ref{betascl1}). This completes the proof.

For part (ii) use the fact that
$\|\Sigma_{SS}g\|^2 \leq s_{max}^2\|g\|^2$ from (\ref{eigenpop}) and  $\|g\| \leq\sigma\sqrt{k}\tilde\maxcon\mu_n$ from (\ref{betalinfconstrain}), to complete the proof.
\end{proof}

\begin{proof}[Proof of Lemma \ref{anallem}]  
We use a result in \citet{szarek1991condition}  
that gives tails bounds for the largest and smallest singular values of Gaussian random matrices. Let $U \in \Rbb^{n\times k}$ be a matrix with i.i.d. standard Gaussian entries. Then, for $r >0$, one has,
$$\PP(\lambda_k\left(U/\sqrt{n}\right) > 1 + \sqrt{k/n} + r) \leq e^{-nr^2/2}$$
$$\PP(\lambda_1\left(U/\sqrt{n}\right) < 1 - \sqrt{k/n} - r) \leq e^{-nr^2/2},$$
where $\lambda_k(.)$ and $\lambda_1(.)$ gives the largest and smallest singular values respectively, of an $n \times k$ matrix. Now, taking $r = \mu_n$, one has, using the above, that with probability at $1 -2/p$  the following holds:
$$h_\ell\|v\|^2 \leq \frac{1}{n}\|Uv\|^2 \leq h_u\|v\|^2 \quad\mbox{for all}\quad v \in \Rbb^k.$$
Now, notice that since $X_S \stackrel{\mathcal{D}}{=} U\Sigma_{SS}^{1/2}$, one has from the above that, with probability at least $1 - 2/p$,
$$h_\ell\|\Sigma_{SS}^{1/2}v\|^2 \leq \frac{1}{n}\|X_Sv\|^2 \leq h_u\|\Sigma_{SS}^{1/2}v\|^2 \quad\mbox{for all}\quad v \in \Rbb^k.$$
Correspondingly, from (\ref{eigenpop}), since $s_{min} \leq \|\Sigma_{SS}^{1/2}v\|^2/\|v\|^2 \leq s_{max}$, which implies that, with probability at least $1 - 2/p$,
$$\lambda_{min}\|v\|^2 \leq \frac{1}{n}\|X_Sv\|^2 \leq \lambda_{max}\|v\|^2 \quad\mbox{for all}\quad v \in \Rbb^k,$$
where $\lambda_{min},\, \lambda_{max}$  as in (\ref{lambdagauss}).

Before proving parts (ii) and (iii), observe that by conditioning on $X_S$, the distribution of $\tilde \epsilon$ may be expressed as, \eqq{\tilde \epsilon \stackrel{\mathcal{D}}{=} X_S g + \sqrt{\ddd}W,\label{tepsiloncond}}
where $g = \Sigma_{SS}^{-1}\Sigma_{SS^c}\beta_{S^c}$ and $\ddd$ as in (\ref{req}). Here $W \sim N(0, I_n)$ and is independent of $X_S$.

For part (ii), notice that from the above $\tilde\sigma^2 := \mbox{Var}(\tilde\epsilon_1)  = \|\Sigma_{SS}g\|^2 + \ddd$, which is at most $\sigma^2(1 + s_{max}^2\tilde\maxcon^2 + \maxcon\bar\sprs)$ from Lemma \ref{concond}. Further, $\|\tilde\epsilon\|^2/\tilde\sigma^2 \sim \Chi^2_n$. Now from Lemma \ref{lem:chitail}, the probability of the event $\|\tilde\epsilon\|^2/(n\tilde\sigma^2) > (1 + \mu_n)^2$ is bounded $1/p$. Use $\mu_n \leq \bar k^{-1/2}$ and $\tilde \sigma^2 \leq \sigma^2(1 + s_{max}^2\tilde\maxcon^2 + \maxcon\bar\sprs)$, to get that $\PP\left(\|\tilde\epsilon\|^2/(n\sigma^2) > \lambda\right) \leq 1/p,$ where $\lambda$ as in (\ref{lambdagauss}).

For part (iii), notice that $\hat \beta_{ls} - \beta_S =(X_S^{\trn}X_S)^{-1}X_S^{\trn}\tilde\epsilon$, which using (\ref{tepsiloncond}), can be expressed as,
\eqq{\hat \beta_{ls} - \beta_S  \stackrel{\mathcal{D}}{=} g + \sqrt{\ddd}(X_S^{\trn}X_S)^{-1}X_S^{\trn}W.\label{lsdiffdist}}
Let $\tilde{\me}_{ls} = \{\sqrt{\ddd}\|(X_S^{\trn}X_S)^{-1}X_S^{\trn}W\|_\infty > \sigma\sqrt{1 + \maxcon\bar\sprs}\tau/\sqrt{\lambda_{min}n}\}$. Now, since $W$ is independent of $X_S$, and $\ddd \leq \sigma^2(1 + \maxcon\bar\sprs)$, one can use the same logic as in the proof of Lemma \ref{maxbdd} to get that,
$\PP(\me_{cond}^c\cap\tilde{\me}_{ls})\leq \sqrt{2/\pi}k/(\tau p^{1+a})$.
Further, $\|g\|_\infty  \leq \sigma\tilde\maxcon\mu_n$ using (\ref{betalinfconstrain}), which, using $\mu_n \leq\tau/\sqrt{n}$, is at most $\sigma\tilde\maxcon\tau/\sqrt{n}$.   Accordingly, on
$\me_{cond}^c\cap\tilde{\me}_{ls}^c$, one has,
\alge{\|\hat \beta_{ls} - \beta_S\|_\infty &\leq \sigma\left[\tilde\maxcon + \sqrt{\frac{1 + \maxcon\bar\sprs}{\lambda_{min}}}\right]\tau/\sqrt{n},\\
&=\sigma \frac{c_0}{\sqrt{n}} \frac{\tau}{1 -\maxcorrlone},}
where $c_0$ as in (\ref{cgauss}). Now use $\tau/(1 - \maxcorrlone) \leq \tau_1$, to get that $\PP(\me_{cond}^c\cap\me_{ls}) \leq \sqrt{2/\pi}k/(\tau p^{1+a})$. This completes the proof of the lemma.
\end{proof}

\begin{proof}[Proof of Lemma \ref{lemconddist}] We first prove part (i). Recall, from (\ref{tepsiloncond}), one has,
$\tilde \epsilon \stackrel{\mathcal{D}}{=} X_S g + \sqrt{\ddd}W$, where $g = (\Sigma_{SS})^{-1}\Sigma_{SS^c}\beta_{S^c}$ and $\ddd$ as in (\ref{req}). Further, $W$ is independent of $X_S$ and follows $N(0, I_n)$.
Correspondingly, the conditional distribution of $X_j$ given $[X_S : W]$ may be expressed as,
$$X_j \stackrel{d}{=} X_S a_j  + b_j W + Z_j$$
where $a_j = \mbox{Cov}(X_{1,S},\, X_{1j})[\mbox{Var}(X_{1,S})]^{-1}$  and $b_j = \mbox{Cov}(X_{1j}, W_1)$. Further, $Z_j \sim N(0, \tilde\sigma_{jj} I_n)$ and is independent of $X_S$ and $W$, with
$$\tilde \sigma_{jj} = \sigma_{jj} - a_j^{\trn}\Sigma_{SS}a_j - b_j^2,$$
which is at most 1.
Clearly, the expression for $a_j$ matches that given in the statement of the lemma. Further, from (\ref{tepsiloncond}), one has that,
$$\mbox{Cov}(X_{1j}, W_1) = \frac{1}{\sqrt{\ddd}}\left[\mbox{Cov}(X_{1j}, \tilde\epsilon_1) - \mbox{Cov}(X_{1j}, X_{1,S}g)\right].$$
Notice that $\mbox{Cov}(X_{1j}, \tilde\epsilon_1) = \Sigma_{jS^c}\beta_{S^c}$ and
$\mbox{Cov}(X_{1j}, X_{1,S}g) = \mbox{Cov}(X_{1j}, X_{1,S}) g$, which is $\Sigma_{jS}\Sigma_{SS}^{-1}\Sigma_{SS^c}\beta_{S^c}$. Correspondingly, the numerator of the above is $e_j^{\trn}\Sigma_{S^c|S}\beta_{S^c}$, and hence, the expression for $b_j$ given above matches that in (\ref{bj}). 

We now prove part (ii) of Lemma \ref{lemconddist}. Firstly, notice that $\max_{j \in S^c}|b_j| \leq \maxcon\mu_n$. This follows from observing that $\ddd \geq \sigma^2$, from (\ref{req}), and also the fact that $|e_j^{\trn} \Sigma_{S^c|S}\beta_{S^c}| \leq \sigma\maxcon\mu_n$, for all $j \in S^c$, from (\ref{betalinfconstrain}).

Recall the statistic $V_{ji}$ given by (\ref{eq:vji}). One sees that,
 \eqq{|V_{ji}| \leq |b_j|\|W\| + \left|E_{ji}\right|. \label{vjibound}}
Now $\|W\|^2 \sim \Chi_n^2$. Correspondingly, from Lemma \ref{lem:chitail}, the event $\{\|W\|/\sqrt{n} > (1 + \mu_n)\}$ has probability at most $1/p$.

Further, $Z_j$'s  are  independent of $[X_S :\tilde\epsilon]$ and, hence, are also independent of $\tilde R_0, \ldots ,\tilde R_m$, since these residuals are functions of $[X_S :\tilde\epsilon]$. Consequently, the $E_{ji}$'s are standard normal random variables; Indeed, conditional on the $\tilde R_i$'s, they follow $N(0,1)$, and hence, follow the same distribution unconditionally. Accordingly, using the same logic as in the proof of Theorem \ref{thm:subgauss}, the event
\eqq{\left\{\max_{1 \leq i \leq m+1,\, j \in S^c}\left|E_{ji}\right| > \tau\right\}\label{eq:maxejiprob}}
has probability bounded by $\sqrt{2/\pi}(k+1)/(\tau p^a)$.

Consequently, using the bounds on $|b_j|$ and the above, one gets that except on a set with probability $1/p +
\sqrt{2/\pi}(k+1)/(\tau p^a)$, one has
$$\max_{1 \leq i \leq m +1,\, j \in S^c} |V_{ji}| \leq \maxcon\mu_n\sqrt{n}\left(1 + \mu_n\right) + \tau.$$ Using $\tau \geq \mu_n\sqrt{n}$ and $\mu_n \leq \bar k^{-1/2}$, the right side of the above is at most $(1 - \maxcorrlone)\tau_1$. This completes the proof of the lemma.
\end{proof}

\section{Conclusion}\label{sec:conclusion}

The paper analyzed variable selection for the OMP for random $X$ matrices.
  We analyzed performance with i.i.d sub-Gaussian designs, which has uses in compressed sensing. We remark that for these i.i.d designs, the analysis carries over for the hard thresholded version of the algorithm, in which, instead of choosing the $j$ which  maximizes the $|\Zcal_{ij}|$'s, one chooses all $j$ satisfying $|\Zcal_{ij}| > \tau$. It is only when there is some correlation within the rows that we find it advantageous to choose the index which maximizes $|\Zcal_{ij}|$.

For Gaussian designs, with correlation within rows, we give much more general results. Apart from showing that results similar to that in \cite{wainwright2009sharp}, for exact support recovery, are also possible using the OMP, we show additional recovery properties by relaxing the assumption of exact sparsity to a more realistic assumption of a control over the $\ell_1$-norm of the smaller coefficients. Oracle inequalities for the coefficient estimate also followed easily as a consequence of these results.

As mentioned earlier, one drawback of the analysis is the crude manner in which the probability of event \eqref{eq:maxejiprob}, that no terms outside of $S$ are selected, is bounded. This  gives rise to the $\sqrt{2/\pi}(k+1)/(\tau p^a)$ term in the expression for $\perrp{k}$ \eqref{prkgauss}, because of which $a$ has to be greater than 1 when $k$ is not negligible compared to $p$. In \cite{fletcherorthogonal}, a more careful analysis had been carried out for exact recovery with i.i.d. designs and $\ell_0$-sparse vectors. We believe that their analysis should carry over for the general case analyzed here, by noting that the random variables $E_{ji}$, for $i = 1,\ldots, m+1$, defined in Lemma \ref{lemconddist}, has the same covariance structure as a normalized Brownian motion at times $t_1,\ldots,\, t_{m+1}$, where $t_i = \|\tilde R_{i-1}\|^2$. This should improve the probability of the event \eqref{eq:maxejiprob} to something closer to $1/p^a$.

For random designs, we measure the performance after averaging over the distribution of $X$. As mentioned before, this can be contrasted to another method, as done in \citet{candès2009near} for the Lasso, in which a distribution is assigned to $\beta$ and the performance is measured after averaging over this distribution. Although these two methods do not imply each other, it is interesting to compare the average performance using both methods. To be consistent with their notation, let's assume that the entries of $X$ are scaled so that the columns have norm equal (or nearly equal) to one.
Under a mild assumption on the incoherence, it is shown that for $\ell_0$-sparse vectors the support can be recovered, if
\eqq{k = O( p/[\ppp X\ppp^2\log p]),\label{candesplancond}}
where $\ppp X\ppp$ denotes the spectral norm of $X$. If $X$ has i.i.d $N(0, 1/n)$ entries, then $\ppp X\ppp \approx \sqrt{p/n}$, so that the sparsity requirement \eqref{candesplancond} would translate to $k = O(n/\log p)$, which is of the same order as what we get here. However, the situation is different in the general case when the rows are i.i.d $N(0, \Sigma/n)$. Then $X$ may be expressed as $\tilde X\Sigma^{1/2}$, where $\tilde X$ has i.i.d $N(0, 1/n)$ entries. Consider the example where $\Sigma_{ii} = 1$ and $\Sigma_{ij} = c/k$, when $i \neq j$, with $c$ appropriately chosen. In this case $\ppp X\ppp \approx c'p/\sqrt{nk}$. Consequently,  \eqref{candesplancond} translates to assuming $n = \Omega(p\,\log p)$. Our results are better in this case, since we only require $\Omega(k\log p)$ observations even for such correlated designs.

An advantage of the work in \cite{candès2009near} is its applicability to broad classes of deterministic designs. It is unclear at this stage whether such results also hold for the OMP.


\appendices

\section{Tail bounds}

A random variable $Z$ is said to be sub-gaussian with mean 0  and scale $\sigma >0$, if $\E e^{tZ} \leq e^{t^2\sigma^2/2}$ for each $t\in \Rbb$.

\begin{lem} \label{subgaus} Let $W = (W_j : 1 \leq j \leq n)^{\trn}$, with each $W_j$  sub-gaussian with mean 0 and scale $\sigma_j > 0$. Let $\sigma = \max_j \{\sigma_j\}$. The following hold.

\begin{enumerate}[(a)]
\item Let $h \in \Rbb^n$, with $\|h\| \leq 1$. If the entries of $W$ are independent then $h^{\trn}W$ is sub-gaussian with mean 0 and scale $\sigma$.
\item Let $\rho = \sigma\sqrt{2(1 + a)\log p}$ with $a > 0 $. Then $\PS(\max_{j}|W_j| > \rho) \leq 2n/p^{1 + a}$.
Further, if the $W_j \sim N(0, \sigma^2)$ then this probability can be bounded by 
$\sqrt{2/\pi}(\sigma n)/(\rho\,p^{1+a})$.
\end{enumerate}
\end{lem}

\begin{proof} For part (a), we need to show that  $\E \exp\{t\, h^{\trn}W\} \leq \exp\{t^2 \sigma^2/2\}$. To see this, notice that $\E \exp\{t\, h^{\trn}W\} = \E\exp\left\{t^2 \sum_{j = 1}^n h_j^2\sigma_j^2/2\right\}$, using independence of $W_j$'s. The claim is proved by noticing that $\sum_{j = 1}^n h_j^2\sigma_j^2/2 \leq \sigma^2$, using $\|h \| \leq 1$ and $\sigma_j \leq \sigma$.

For part (b), use a Chernoff bound, followed by optimizing the exponent to get that,
$$\PP(|W_j| > \rho) \leq 2\exp\left(-\frac{\rho^2}{2\sigma^2}\right).$$
If the $W_j$'s were normal, standard tail bounds \cite{feller1950introduction} reveals that the above bound can be improved to $(2/(\sqrt{2\pi}\rho))\exp\left(-\frac{\rho^2}{2\sigma^2}\right)$. Now use a union bound, along with the fact that $\exp\left(-\frac{\rho^2}{2\sigma^2}\right) = 1/p^{1+a}$, to prove the claim.
\end{proof}

Next we give a simple lemma on chi-square tail bounds, which will be used repeatedly.

\begin{lem}\label{lem:chitail} Let $W$ follow $N(0,\, I_n)$. Then
\eqq{\PP\left(\|W\|/\sqrt{n} \geq  1 + \mu_n\right)\leq 1/p \label{chitail},}
where $\mu_n = \sqrt{(2\log p)/n}$.
\end{lem}
\begin{proof} Use the fact (see for example \cite{donoho2006most}) that for $h >0$, one has $$\PP\left(\|W\|/\sqrt{n} \geq  1 + h\right)\leq e^{-nh^2/2}.$$
Substitute $h = \sqrt{(2\log p)/n}$ to get the result.
\end{proof}

\section{Proof of Lemma \ref{zhnglem}} \label{pf:zhnglem}

For convenience, let $S = \{1,\ldots,\, k\}$. 
 Let $H_j,\, 1\leq j \leq k$ denote the columns of the $H$ matrix. Assume that the algorithm runs for $m$ steps and let $R_1,\ldots, R_{m-1}$ denote the associated residuals. Let $R_0 = Y$.
 Denote as
$\fit_{\mathcal{A}}$, the least square fit when $U$ is regressed on $H_{\mathcal{A}}$.
We also denote as $u(i) = S - d(i)$, which corresponds to the terms in $S$ undetected after step $i$. We assume $u(0) = S$ and $\fit_{d(0)} = 0$. 

The following lemma is from \citet{zhang2009consistency}.
\begin{lem}\label{zhanglem}(\citet{zhang2009consistency}) For each $i$, with $0 \leq i < m$, if $|u(i)| > 0$, then
$$\max_{j \in u(i)}\left|\frac{H_j^{\trn}R_i}{\|H_j\|}\right| \geq \sqrt{\lambda_{min}}\frac{\|\fit_{d(i)} - \fit_S\|}{\sqrt{|u(i)|}} ,$$
\end{lem}

The results is a consequence of Lemmas 6 and 7 in \citet[page 566]{zhang2009consistency}.
Using his notation, in our case, $\lambda_{min} =\rho(\bar F) ,\, R_i = Y - X\beta^{(k-1)},\, \fit_{d(i)} = X\beta^{(k-1)},\, \fit_S = X\beta_X(\bar F,\, y)$ and $u(i) = \bar F - F^{(k-1)}$.

\begin{lem} \label{resbdd} For each $i$, with $0\leq i \leq m$, one has $$\|R_i\|/\sqrt{n} \leq \sqrt{\tilde\lambda_{max}}(\|\varphi_{u(i)}\| + \sigma),$$
where $\tilde\lambda_{max} = \max\{\lambda,\, \lambda_{max}\}$.
\end{lem}

\begin{proof}[Proof of \ref{resbdd}] Write $R_i = (I - \mathcal{P}_i)U$, where here $\mathcal{P}_i$ is the projection matrix for column space of $H_{d(i)}$. Now $U = H_{d(i)}\varphi_{d(i)} + H_{u(i)}\varphi_{u(i)} + \epsilon$ and
$(I - \mathcal{P}_i)H_{d(i)} =0$. Correspondingly, $R_i = (I - \mathcal{P}_i)[ H_{u(i)}\varphi_{u(i)} + \epsilon]$. Consequently,  $\|R_i\| \leq \|H_{u(i)}\varphi_{u(i)}\| + \|\epsilon\|$, since $\|(I - \mathcal{P}_i)x\| \leq \|x\|$ for any $x \in \Rbb^n$. The result immediately follows from using $\|H_{u(i)}\varphi_{u(i)}\|/\sqrt{n} \leq \sqrt{\lambda_{max}}\|\varphi_{u(i)}\|$
and $\|\epsilon\|/(\sqrt{n}\sigma) \leq \sqrt{\lambda}$. This completes the proof of the lemma.
\end{proof}
Now use the fact that $\|H_j\| \geq \sqrt{n}\sqrt{\lambda_{min}}$, to get from Lemma \ref{zhanglem} that,
$$\max_{j \in u(i)}\left|H_j^{\trn}R_i\right| \geq \sqrt{\frac{n\rho_1}{|u(i)|}} \|\fit_{d(i)} - \fit_S\|,$$
where $\rho_1 = \lambda_{min}^2$.
Consequently, using  Lemma \ref{resbdd} and the above, one has that,
$$\max_{j \in u(i)}\left|H_j^{\trn}\frac{R_i}{\|R_i\|}\right| \geq \sqrt{ \frac{n\rho_2}{|u(i)|}} \frac{\|\fit_{d(i)} - \fit_S\|/\sqrt{n}}{\|\varphi_{u(i)}\| + \sigma},$$
where $\rho_2 = \rho_1/\tilde{\lambda}_{max}$.
The algorithm continues as long as the left side of the above is at least $\tau_0$. Consequently, following the reasoning in \cite{zhang2009consistency},
when the algorithm stops, one must have that either $|\hat F_2| = 0$ or the right side of the above, with $u(i)$ replaced by $\hat F_2$, is at most $\tau_0$. Let's assume that $|\hat F_2| > 0$, since otherwise we would have correctly decoded all terms. Correspondingly, we have,
\eqq{\|\fit_{\hat S} - \fit_S\|/\sqrt{n} \leq \tau_0 \sqrt{\frac{|\hat F_2|}{n\rho_2}}(\|\varphi_{\hat F_2}\| + \sigma)\label{algostop}}
when the algorithm stops. Now, \eqq{\|\varphi_{\hat F_2}\| \leq \sqrt{|\hat F_2|}\|\varphi - \hat \varphi_{ls}\|_{\infty} + \|\hat \varphi_{ls} - \hat \varphi\|.\label{residsignal}} To see this note that $\|\varphi_{\hat F_2}\|$ is bounded by the sum of $\|\varphi_{\hat F_2} - \hat \varphi_{ls,\,\hat F_2}\|$ and $\|\hat \varphi_{ls,\,\hat F_2}\|$, where $\hat\varphi_{ls,\,\hat F_2}$ is the sub-vector of $\hat\varphi_{ls}$ with indices in $\hat F_2$. The first term in the bound is at most $\sqrt{|\hat F_2|}\|\varphi - \hat \varphi_{ls}\|_{\infty}$, whereas the second term can be bounded by $\|\hat \varphi_{ls} - \hat \varphi\|$, since $\hat \varphi_j$ is zero for all indices $j$ in $\hat F_2$. Now, use the fact that  $\|\hat \varphi_{ls} - \varphi\|_{\infty}$ is bounded by $c_0\sigma\tau_0/\sqrt{n}$ along with the fact that $\|\fit_{\hat S} - \fit_S\|/\sqrt{n} \geq \sqrt{\lambda_{min}}\| \hat \varphi - \hat \varphi_{ls}\|$, to get that from (\ref{algostop}) and \eqref{residsignal} that,
\eqq{\|\varphi_{\hat F_2}\| \leq c_0\sigma\tau_0\sqrt{\frac{|\hat F_2|}{n}} + \tau_0\sqrt{r_1\frac{|\hat F_2|}{n}}(\|\varphi_{\hat F_2}\| + \sigma)\label{betnrmrem}}
when the algorithm stops. Here we use that $r_1 = 1/(\lambda_{min}\rho_2)$. One gets from (\ref{betnrmrem}) that
 \eqq{\left(1 - \tau_0\sqrt{\frac{r_1 |\hat F_2|}{n}}\right)\|\varphi_{\hat F_2}\| \leq \trt\sigma\tau_0\sqrt{|\hat F_2|}/\sqrt{n},\label{phihatfbdd}}
  where $\trt = c_0 + \sqrt{r_1}$ and $r_1 = 1/\rho$. Using $|\hat F_2| \leq k$, the term $ \tau_0\sqrt{r_1|\hat F_2|/n}$ appearing in the left side of the above can be bounded by $\tau_0\sqrt{r_1 k/n}$. This leads us to (\ref{stopineq}), which completes the proof of part (a).

For part (b),  notice that
\eqq{\|\hat\varphi - \varphi\| \leq \sqrt{k}\|\hat\varphi_{ls} - \varphi\|_\infty + \|\hat \varphi_{ls} - \hat\varphi\|\label{l2normtruncbdd}.}
Now use,
$$\|\hat \varphi_{ls} - \hat \varphi\| \leq \tau_0\sqrt{r_1k/n}(\|\varphi_{\hat F_2}\| + \sigma)$$
along with,
 \eqq{\|\varphi_{\hat F_2}\| \leq \frac{\trt\sigma\tau_0\sqrt{k/n}}{\left(1 - \tau_0\sqrt{r_1k/n}\right)},\label{phihatfbdd}}
to get, after rearranging, that,
$$\|\hat \varphi_{ls} - \hat \varphi\| \leq \sigma\tau_0\sqrt{r_1k/n}\,\frac{ (c_0\tau_0\sqrt{k/n} +1)}{1 - \tau_0\sqrt{r_1k/n}}.$$
Now use $\|\hat\varphi_{ls} - \varphi\|_\infty \leq \sigma c_0\tau_0\sqrt{k/n}$, along with $\trt = c_0 +\sqrt{r_1}$, to get from (\ref{l2normtruncbdd}) and the above that,
$$\|\hat\varphi - \varphi\| \leq  \frac{\trt\sigma\tau_0\sqrt{k/n}}{\left(1 - \tau_0\sqrt{r_1k/n}\right)}.$$
This completes the proof of the lemma.

\section{Proof of Lemma \ref{lemsimpsuffcond}}\label{pf:lemsimpsuffcond}
For a matrix $A \in \Rbb^{n\times m}$, and $a = 1$ or $\infty$, denote as $\ppp A \ppp_a = \sup_{v\neq 0} \|Av\|_a/\|v\|_a$.
Recall that $\ppp A\ppp_1$ is the maximum of the $\ell_1$ norms of the columns, whereas $\ppp A\ppp_\infty $ is the maximum of the $\ell_1$ norms of the rows.

We first prove part (i). We use \citet[Lemma 2]{cai2010orthogonal}, to get that $$1 - \gamma(k-1) \leq s_{min} \leq s_{max} \leq 1 + \gamma(k-1).$$
Now $\gamma \leq \maxcorr/(2k)$, since $k \leq \bar k$, and hence, the left side of the above is at least $1 - \maxcorr/2$ and the right side is at most $1 + \maxcorr/2$. Further, use  \citet[Theorem 3.5]{tropp2004greed}, to get that
$$\|\Sigma_{SS}^{-1}\Sigma_{Sj}\|_1 \leq \frac{\gamma k}{1 - \gamma(k-1)}.$$
The right side of the above is at most $\maxcorr$. Correspondingly, we may take $\maxcorrlone$ as $\maxcorr$.

We next prove part (ii). Use the fact that,
\eqq{\| \Sigma_{SS}^{-1}\Sigma_{SS^c}\beta_{S^c}\|_\infty \leq \ppp \Sigma_{SS}^{-1}\ppp_\infty \|\Sigma_{SS^c}\beta_{S^c}\|_\infty.\label{bddtkappa}}
Now as $\Sigma_{SS}^{-1}$ is symmetric, $\ppp \Sigma_{SS}^{-1}\ppp_\infty = \ppp \Sigma_{SS}^{-1}\ppp_1$; the latter is at most $1/(1 - \gamma(k-1))$ from  \cite[Theorem 3.5]{tropp2004greed}. Further, $\|\Sigma_{SS^c}\beta_{S^c}\|_\infty \leq \gamma\|\beta_{S^c}\|_1$, which is at most $\sigma\gamma\sprs\mu_n$. Correspondingly, from (\ref{bddtkappa}), one gets
\eqq{\|\Sigma_{SS}^{-1}\Sigma_{SS^c}\beta_{S^c}\|_\infty \leq \sigma\frac{\gamma \bar k}{1 - \gamma (k-1)}\bar\sprs\mu_n.\label{eq:ssusscbsc}}
The right of the above is at most $\sigma\maxcorr\bar\sprs\mu_n$, using the bound on $\gamma$.
Further,
\eqq{\|\Sigma_{S^c|S}\|_\infty \leq \|\Sigma_{S^cS^c}\beta_{S^c}\|_\infty +\|\Sigma_{S^cS}\Sigma_{SS}^{-1}\Sigma_{SS^c}\beta_{S^c}\|_\infty.\label{bddkappa}}
Now, $\|\Sigma_{S^cS^c}\beta_{S^c}\|_\infty \leq \|\beta_{S^c}\|_\infty + \|(\Sigma_{S^cS^c} - I)\beta_{S^c}\|_\infty.$ Further, use $\|\beta_{S^c}\|_\infty \leq \sigma\maxinf\mu_n$ and $\|(\Sigma_{S^cS^c} - I)\beta_{S^c}\|_\infty \leq \gamma\|\beta_{S^c}\|_1$, the right side of which is at most $\sigma\gamma \sprs\mu_n$. Also, the second term in (\ref{bddkappa}) can be bounded as follows:
$$\|\Sigma_{S^cS}\Sigma_{SS}^{-1}\Sigma_{SS^c}\beta_{S^c}\|_\infty \leq \ppp \Sigma_{S^cS}\ppp_\infty \|\Sigma_{SS}^{-1}\Sigma_{SS^c}\beta_{S^c}\|_\infty.$$
The first term in the right side product is bounded by $\gamma k$, whereas the second term, from \eqref{eq:ssusscbsc}, is bounded by $\sigma\maxcorr\bar\sprs\mu_n$. Correspondingly, one gets that
$$\|\Sigma_{S^c|S}\beta_{S^c}\|_\infty \leq \sigma\maxinf\mu_n + \sigma\gamma\sprs\mu_n + \sigma\gamma\maxcorr\sprs\mu_n.$$
Further, using $\gamma \sprs + \gamma \sprs \maxcorr \leq 2\gamma \sprs$, which is at most $\maxcorr\bar\sprs$, one gets the bound on
$\|\Sigma_{S^c|S}\beta_{S^c}\|_\infty$.

For $k = 0$, one has $\|\Sigma_{S^cS^c}\beta_{S^c}\|_\infty \leq \maxinf + \maxcorr\bar\sprs$, which is at most $\maxinf + \maxcorr\bar\sprs$, from the bound derived above. This completes the proof of the lemma.

\bibliographystyle{abbrv}
\bibliography{ogarandom0819}

\end{document}